\documentclass[12pt]{article}
\usepackage[letterpaper,left=1in,bottom=1in,top=1in,right=1in]{geometry}

\usepackage[utf8]{inputenc}
\usepackage[T1]{fontenc}
\usepackage[english]{babel}

\usepackage{microtype}
\usepackage{csquotes}

\usepackage{setspace}
\usepackage[en-GB]{datetime2}
\DTMlangsetup[en-GB]{showdayofmonth=false}

\usepackage{titlesec}

\titleformat{\part}[frame]
{\normalfont}
{\filright \footnotesize \enspace PART \thepart \enspace}
{8pt}
{\vspace{20pt} \huge \bfseries \filcenter}

\titlespacing*{\part}{0pt}{-50pt}{10pt}

\usepackage{titletoc}

\usepackage{enumitem}

\usepackage{graphicx}

\usepackage{tikz}
\usetikzlibrary{arrows.meta, positioning}

\usepackage{caption}
\usepackage{subcaption}
\usepackage{float}

\usepackage{xcolor}

\usepackage{amsmath}
\usepackage{amssymb}
\usepackage{mathtools}
\usepackage{amsthm}
\usepackage{thmtools}
\usepackage{etoolbox}

\usepackage{amsfonts}
\usepackage{mathrsfs}
\usepackage{dsfont}
\usepackage{bm}

\usepackage{relsize}
\usepackage{nicefrac}
\usepackage{centernot}

\usepackage{algorithm}

\usepackage{algpseudocode}
\usepackage{etoolbox}

\theoremstyle{plain}
\newtheorem{theorem}{Theorem}[section]
\newtheorem{lemma}[theorem]{Lemma}

\newtheorem{corollary}[theorem]{Corollary}

\theoremstyle{definition}
\newtheorem{definition}[theorem]{Definition}

\theoremstyle{remark}

\AtBeginEnvironment{proof}{\setcounter{localclaim}{0}}



\DeclarePairedDelimiter{\ceil}{\lceil}{\rceil}

\usepackage{tcolorbox}

\tcbset{
    rounded corners,
    colback = white,
    before skip = 0.25cm,
    after skip = 0.5cm,
    boxrule = 2pt,
    arc = 10pt
}

\newtcolorbox{boxtheorem}{
  colback=white,
  colframe=black,
  fonttitle=\bfseries,
  title=Theorem}

\newcommand{\figref}[1]{\hyperref[#1]{Figure~\ref*{#1}}}

\newcommand{\DS}{\mathrm{DS}}

\usetikzlibrary{arrows.meta, positioning, calc}

\renewcommand{\Pr}{\operatorname{{Pr}}}

\newcommand{\listC}{\mathcal C}

\newcommand{\cA}{\mathcal A}

\newcommand{\cD}{\mathcal D}

\newcommand{\cG}{\mathcal G}
\newcommand{\cH}{\mathcal H}

\newcommand{\cX}{\mathcal X}
\newcommand{\cY}{\mathcal Y}

\newcommand{\dexp}{d^{\ell}_{\operatorname{E}}}

\newcommand{\dds}{d^{\ell}_{\operatorname{DS}}}

\newcommand{\avd}{\mathsf{avd}^\ell}
\newcommand{\savd}{\mathsf{savd}^\ell}

\newcommand{\outdeg}{\mathsf{outdeg}^\ell}
\newcommand{\degree}{\mathsf{deg}^\ell}

\newcommand{\Gl}{\mathtt{G}_\ell}

\newcommand{\sigmal}{\sigma^{\ell}}
\newcommand{\mdl}{\mathrm{MD}^{\ell}}

\newcommand{\ole}{\mathrm{E}_{\ell}}
\newcommand{\olds}{\operatorname{DS}_{\ell}}

\newcommand{\shift}{\mathbb{S}}

\usepackage[numbers]{natbib}
\usepackage[framemethod=tikz]{mdframed}
\newmdenv[
  linewidth=0.8pt,
  roundcorner=4pt,
  skipabove=10pt,
  skipbelow=10pt,
  innerleftmargin=8pt,
  innerrightmargin=8pt,
  innertopmargin=0pt,
  innerbottommargin=6pt
]{theoremframe}

\usepackage[colorlinks=true, linkcolor=red]{hyperref}
\hypersetup{
citecolor=[rgb]{0.0,0.3,0.0},
linktocpage=true,
pdftitle={A Sharp Sauer Lemma},
pdfpagemode=FullScreen
}

\usepackage{xurl}

\usepackage[capitalize,noabbrev]{cleveref}

\title{An Optimal Sauer Lemma Over $k$-ary Alphabets}

\author{ 
Steve Hanneke\thanks{Purdue University, USA. Email: \texttt{steve.hanneke@gmail.com}} \and
Qinglin Meng\thanks{Purdue University, USA. Email: \texttt{meng160@purdue.edu}} \and
Shay Moran\thanks{Technion, Israel and Google Research. Email: \texttt{smoran@technion.ac.il}} \and
Amirreza Shaeiri\thanks{Purdue University, USA. Email: \texttt{amirreza.shaeiri@gmail.com}}
}

\begin{document}

\maketitle

\thispagestyle{empty}

\begin{abstract}
The Sauer-Shelah-Perles Lemma is a cornerstone of combinatorics and learning theory, bounding the size of a binary hypothesis class in terms of its Vapnik-Chervonenkis (VC) dimension. For classes of functions over a $k$-ary alphabet, namely the multiclass setting, the Natarajan dimension has long served as an analogue of VC dimension, yet the corresponding Sauer-type bounds are suboptimal for alphabet sizes $k>2$.

In this work, we establish a sharp Sauer inequality for multiclass and list prediction. Our bound is expressed in terms of the Daniely--Shalev-Shwartz (DS) dimension, and more generally with its extension, the list-DS dimension -- the combinatorial parameters that characterize multiclass and list PAC learnability. Our bound is tight for every alphabet size $k$, list size $\ell$, and dimension value, replacing the exponential dependence on $\ell$ in the Natarajan-based bound by the optimal polynomial dependence, and improving the dependence on $k$ as well. Our proof uses the polynomial method. In contrast to the classical VC case, where several direct combinatorial proofs are known, we are not aware of any purely combinatorial proof in the DS setting. This motivates several directions for future research, which are discussed in the paper.

As consequences, we obtain improved sample complexity upper bounds for list PAC learning and for uniform convergence of list predictors, sharpening the recent results of Charikar et al.~(STOC~2023), Hanneke et al.~(COLT~2024), and Brukhim et al.~(NeurIPS~2024).

\end{abstract}

\setcounter{page}{1}

\section{Introduction}

The Sauer--Shelah--Perles Lemma is one of the foundational results of extremal combinatorics and statistical learning theory. It states that any family of binary vectors of length $n$ with VC dimension~$d$ has cardinality at most
\(\sum_{i=0}^d \binom{n}{i}.\) 
This inequality lies at the heart of statistical learning theory, where it underpins the characterization of PAC learnability by the VC dimension. It also plays a central role in applications of VC theory across computer science, discrete mathematics, and probability theory.

A natural question is how this inequality extends beyond the binary setting. When the label space has size $k > 2$, the classical VC dimension no longer captures the learnability. Natarajan introduced a multiclass analogue of the VC dimension in the 1980s, and Sauer-type inequalities in terms of the Natarajan dimension have since been widely used. However, these bounds are not tight for $k > 2$. 
As we discuss in greater detail later, these suboptimal combinatorial bounds propagate to correspondingly loose guarantees on learning rates and sample complexities.

Recent work has established that the correct parameter governing multiclass and list PAC learnability is the Daniely--Shalev-Shwartz (DS) dimension. It is therefore natural to ask whether sharper Sauer-type bounds can be obtained in terms of the DS dimension, and whether such bounds lead to improved quantitative guarantees for learning. In this work we establish a sharp Sauer inequality for $k$-ary hypothesis classes in terms of the DS dimension, and more generally, its list extension, the $\ell$-DS dimension. The resulting bound is tight for all admissible parameter values, and improves previously known bounds based on the Natarajan dimension.

\paragraph{Binary Alphabets.}
We now recall the classical Sauer–Shelah–Perles Lemma more formally, before discussing its multiclass extensions. Let $\mathcal{H} \subseteq \{0,1\}^n$ be a class of binary vectors. 
A set $\{i_1,\dots,i_d\} \subseteq [n]$ of indices is \emph{shattered} by $\mathcal{H}$ if
\[
\{(h(i_1),\dots,h(i_d)) : h \in \mathcal{H}\} = \{0,1\}^d,
\]
that is, if the restriction of $\mathcal{H}$ to $\{i_1,\dots,i_d\}$ realizes the entire $d$-dimensional Boolean cube.
The Vapnik--Chervonenkis (VC) dimension of $\mathcal{H}$ is the maximum size of a shattered set.
The Sauer--Shelah--Perles Lemma states that if $\mathcal{H}$ has VC dimension $d$, then 
\[
|\mathcal{H}| 
\le \sum_{i=0}^{d} \binom{n}{i}.
\]
The lemma was independently proved by Vapnik and Chervonenkis in \cite*{VC71}, Sauer in \cite*{Sau72}, and Shelah in \cite*{She72}; Shelah explicitly credits Micha Perles for an unpublished proof, which motivated the name ``Sauer--Shelah--Perles''.

The Sauer Lemma admits a remarkable variety of proofs. There are direct inductive arguments, combinatorial proofs via shifting and symmetrization, and linear-algebraic proofs. Gil Kalai is quoted as describing the lemma as unusually accessible: \emph{whichever natural approach one takes, a proof seems to emerge}. We will return to this point when discussing our extension, for which we presently know only a proof based on an indirect algebraic argument.

\subsection*{Beyond Binary Alphabets}
Extending the Sauer Lemma to larger alphabets has been studied extensively over the past several decades. 
In the 1980s, Natarajan introduced a multiclass analogue of the VC dimension~\cite*{natarajan1989learning}, 
with the goal of extending the VC characterization of PAC learnability from binary to multiclass settings.

\paragraph{Cubes and the Natarajan dimension.}
Let $Y$ be a finite alphabet, which without loss of generality we identify with $[k]=\{0,1,\ldots,k-1\}$. 
For $\ell \le k$, a subset $B \subseteq [k]^d$ is called an \emph{$\ell$-cube} if it is a Cartesian product of $\ell$-subsets of $[k]$; that is,
\(B = Y_1 \times \cdots \times Y_d,\)
where each $Y_i \subseteq [k]$ satisfies~$|Y_i| = \ell$. 
For a class $\mathcal{H} \subseteq [k]^n$, 
we say that a set $\{i_1,\ldots,i_d\} \subseteq [n]$ is \emph{$\ell$-Natarajan-shattered} if
\[\{(h(i_1),\ldots,h(i_d)) : h \in \mathcal{H}\}\]
contains an $\ell$-cube. 
The \emph{$\ell$-Natarajan dimension} of $\mathcal{H}$ is the maximum size of a set that is $(\ell+1)$-Natarajan-shattered. 
This recovers several standard notions: when $\ell=1$ one obtains the classical Natarajan dimension, when $k=2$ and $\ell=1$ this is exactly the VC dimension, and when $\ell=k-1$ it coincides with the Steele dimension~\cite{Ste78}.

Natarajan showed that the Natarajan dimension characterizes PAC learnability when the label space $Y$ is bounded. 
Whether this characterization extends to unbounded label spaces remained open until recently, when Brukhim et al.~\cite{brukhim2022characterization} showed that it does not. 
They further proved that the dimension characterizing PAC learnability in full generality is the Daniely--Shalev--Shwartz (DS) dimension, introduced in~\cite{daniely2014optimal}, which underlies the main result of this paper. 
We defer its formal definition to the next section.

In the remainder of this section, we discuss Sauer-type bounds based on the Natarajan dimension.
Such bounds have been extensively studied in learning theory but also in other areas, including combinatorics, geometry, 
game theory, dynamical systems, model theory, and complexity theory; see, e.g., 
\cite*{Ste78, KM78, Alo83, natarajan1989learning, HL95, fischer-matousek-2001-natarajan, hegedus2010multivalued, furedi2012optimal, DanielySS18, charikar2023characterization, gao2023complexity}.
The Sauer-type inequality in this setting asserts that if $\mathcal{H} \subseteq Y^n$ has $\ell$-Natarajan dimension $d$, then
\[
|\mathcal{H}| 
\le 
\ell^{\,n-d} 
\sum_{i=0}^{d} 
\binom{n}{i} 
\binom{k}{\ell+1}^{\,i}.
\]
Up to lower-order terms, this yields
\begin{equation}
\label{eq:natarajan-sauer}
|\mathcal{H}|
\lesssim
\ell^{\,n-d}\; n^{d}\; k^{(\ell+1)d}.
\end{equation}
This inequality is due to~\cite{DanielySS18,charikar2023characterization}, 
extending earlier work of~\cite{HL95} in the case $\ell=1$. 
More recently,~\cite{gao2023complexity} obtained an improved but more intricate bound with the same asymptotic dependence.

The bound is tight only when $k=2$, in which case it coincides with the classical Sauer-Shelah-Perles Lemma. For larger alphabets, it is not tight. The following two examples illustrate this gap and help motivate the sharper inequality established in this work.

\medskip
\noindent\textbf{Example 1.}
Let $\mathcal{H} \subseteq [k]^n$ consist of all vectors $v$ satisfying the following property: 
for every~$v \in \mathcal{H}$, we have $v(i) < \ell$ for all but at most $d$ coordinates $i \in [n]$. 
Equivalently, each vector may take arbitrary values in $[k]$ on at most $d$ coordinates, 
and is restricted to $\{0,1,\ldots,\ell-1\}$ on the remaining coordinates. 
The $\ell$-Natarajan dimension of $\mathcal{H}$ is $d$, and a direct counting argument shows that
\[
|\mathcal{H}|
=
\sum_{i=0}^{d} \binom{n}{i} (k-\ell)^{\,i} \ell^{\,n-i}.
\]
In particular, for fixed $d$, this behaves as
\[
|\mathcal{H}| \approx \ell^{\,n-d} \, n^d \, (k-\ell)^d.
\]
Compared to~\eqref{eq:natarajan-sauer}, this replaces the factor $k^{(\ell+1)d}$ by $(k-\ell)^d$, 
which is exponentially smaller in $\ell$ and polynomially smaller in $k$.

As we will show later, our inequality in terms of the DS dimension matches this behavior. 
In fact, this example will turn out to be extremal for the DS-based bound, thereby witnessing its sharpness.~\footnote{Note that for $k = 2$ and $\ell = 1$, this example also shows the tightness of the classical Sauer–Shelah–Perles Lemma.}

\medskip
\noindent\textbf{Example 2.}
Consider the case $n=2$.
A class $\mathcal{H} \subseteq [k]^2$ can be identified with the edge set of a $k \times k$ bipartite graph. 
Having $\ell$-Natarajan dimension equal to $1$ is equivalent to the graph that does not contain a copy of $K_{\ell+1,\ell+1}$, the complete bipartite graph with sides of size $\ell+1$. 
Thus, upper bounding $|\mathcal{H}|$ in this case becomes a classical Turán-type problem: 
how many edges can a $k \times k$ bipartite graph have if it contains no~$K_{\ell+1,\ell+1}$? The Kővári-Sós-Turán theorem~\cite{kovari1954problem} implies that
\[
|\mathcal{H}| = O\!\left(k^{2-\frac{1}{\ell+1}}\right),
\]
and this bound is asymptotically tight for $\ell \le 2$. 
Whether it is tight for larger $\ell$ is a long-standing and notoriously difficult open problem in extremal combinatorics (see, e.g.~\cite{bukh2015random}).
By contrast, the bound~\eqref{eq:natarajan-sauer} in this case yields only the trivial upper bound $|\mathcal{H}| = O(k^2)$.

\smallskip

These examples illustrate two limitations of the Natarajan-based bound.
The first suggests that the exponential dependence on $\ell$ in~\eqref{eq:natarajan-sauer} may be unnecessarily pessimistic. 
The second example indicates that tight bounds purely in terms of the $\ell$-Natarajan dimension may be difficult to obtain, as even the case $n=2$ reduces to longstanding extremal problems.
We will show that the bound can be improved when parameterizing by the DS dimension, 
yielding a sharp inequality and, as a consequence, substantial improvements in PAC sample complexity bounds.

\section{The Inequality}
We begin by defining the Daniely-Shalev-Shwartz (DS) dimension. 
The key combinatorial object underlying this notion is the concept of a pseudo-cube, which generalizes the notion of a cube (recall that an $\ell$-cube is a Cartesian product of subsets of $[k]$ of size $\ell$).

\paragraph{Pseudo-cubes and the DS dimension.}
A nonempty set $B \subseteq [k]^d$ is called an \emph{$\ell$-pseudo-cube} if for every $v \in B$ and every coordinate $i \in [d]$, there exist distinct vectors $v_1,\ldots,v_{\ell-1} \in B$ that agree with $v$ on all coordinates except $i$. 
Thus, every vector in an $\ell$-pseudo-cube has at least $\ell-1$ neighbors in every coordinate direction.

Every $\ell$-cube is clearly an $\ell$-pseudo-cube, but the converse fails when $d>1$. For example, $B = \{ (1,2), (3,2), (3,4), (5,4), (5,6), (1,6) \}$ is not isomorphic to the Boolean cube, while it is indeed a $2$-pseudo-cube.
For $d=2$, a set $B \subseteq [k]^2$ may be viewed as the edge set of a bipartite graph on $[k]\sqcup[k]$. 
In this interpretation, $B$ is an $\ell$-pseudo-cube precisely when every non-isolated vertex has degree at least $\ell$. 
In contrast, an $\ell$-cube corresponds to the special case of a complete bipartite subgraph~$K_{\ell,\ell}$ of the ambient $k \times k$ graph.
For larger $d$, the combinatorial structure of pseudo-cubes becomes richer and connects to objects from topology and algebra; see~\cite{brukhim2022characterization}.

Let $\mathcal{H} \subseteq [k]^n$. 
We say that a set $\{i_1,\ldots,i_d\} \subseteq [n]$ is \emph{$\ell$-DS-shattered} by $\mathcal{H}$ if
\[
\{(h(i_1),\ldots,h(i_d)) : h \in \mathcal{H}\}
\]
contains an $(\ell+1)$-pseudo-cube. 
The \emph{$\ell$-DS dimension} of $\mathcal{H}$ is the maximum size of a set that is $\ell$-DS-shattered. 
Since every $(\ell+1)$-cube is an $(\ell+1)$-pseudo-cube, the $\ell$-DS dimension is always at least the $\ell$-Natarajan dimension. Moreover, in the case of $k = 2$ and $\ell = 1$, the $\ell$-DS dimension equals the VC dimension.

\medskip

The $(\ell=1)$-DS dimension was introduced by Daniely and Shalev--Shwartz~\cite*{daniely2014optimal} as a candidate dimension characterizing PAC learnability in the multiclass setting. 
Brukhim et al.~\cite*{brukhim2022characterization} later confirmed that the DS dimension indeed characterizes multiclass PAC learnability. 
Subsequently, Charikar et al.~\cite{charikar2023characterization} showed that, more generally, the $\ell$-DS dimension characterizes list-PAC learning with lists of size $\ell$, in which the learner outputs, for each input $x$, a list of at most $\ell$ labels and succeeds if the true label lies in this list. This list-learning framework arises naturally in applications such as recommendation systems, where a user is presented with a short list of options and success requires that at least one of them is correct. 
Beyond its practical motivation, list learning has also emerged as a useful intermediate tool in multiclass prediction, serving as a label-space reduction technique \cite{brukhim2022characterization}. 
It has been studied extensively in recent years; see, e.g., 
\cite{brukhim2022characterization,moran2023list,brukhim2023improperboosting,brukhim2024multiclass,hanneke2024list,hanneke2024improved,hanneke2025private,pabbaraju2024characterizationlistregression,charikar2025limit,erez2025list,cohen2026natarajan}.

\medskip

We are now ready to state the main combinatorial inequality of the paper.

\begin{theoremframe}
\begin{theorem}[A Sauer inequality for the Daniely--Shalev-Shwartz dimension]\label{thm:main}
Let $\mathcal{H} \subseteq [k]^n$ and let $d$ denote the $\ell$-DS dimension of $\mathcal{H}$. 
Then
\[
|\mathcal{H}|
\le
\sum_{i=0}^{d}
\binom{n}{i}
(k-\ell)^{\,i}
\ell^{\,n-i}.
\]
\end{theorem}
\end{theoremframe}


This bound coincides exactly with the size of the class in Example~1. 
Since that class has $\ell$-DS dimension $d$, the inequality is tight for every admissible choice of $n,k,\ell$, and $d$. In the special case $\ell=1$, the bound reduces to
\[
|\mathcal{H}|
\le
\sum_{i=0}^{d}
\binom{n}{i}
(k-1)^{\,i},
\]
which recovers the classical Sauer--Shelah--Perles inequality when $k=2$ and replaces the $\binom{k}{2}$ term in the one-dimensional Natarajan bound by the linear factor $(k-1)$.

The tightness of the bound and its dependence on $\ell$ and $k$ stand in contrast to the inequalities in terms of the $\ell$-Natarajan dimension discussed in the previous section. 
In particular, as seen there, Natarajan-based bounds are not tight in general and exhibit exponential dependence on $\ell$, whereas the inequality above achieves tight growth and replaces this by a polynomial dependence. 
This improvement in the dependence on $\ell$ plays a central role in the learning-theoretic applications presented in the next section.

\paragraph{Organization of the remainder of the paper.}
We begin with the proof of the main Sauer-type inequality for the DS dimension. 
The subsequent section discusses several open problems and directions for further investigation, both in combinatorics and in learning theory. 
The final section turns to learning-theoretic implications, where we demonstrate concrete applications of the inequality to multiclass PAC learning.

The discussion section may be read independently of the proof, and the applications section relies only on the statement of the inequality.
Readers primarily interested in learning-theoretic consequences may therefore proceed directly to that section.
The proof itself may nevertheless be of independent interest as a particularly transparent application of the polynomial method in extremal combinatorics. For additional applications of linear-algebraic methods in combinatorics, see the textbook of Babai and Frankl~\cite{babai2020linear}.
\subsection{Proof}
\begin{proof}[Proof of \Cref{thm:main}]
The proof proceeds by a linear-algebraic argument. 
At its core lies the following elementary principle: in any vector space, the size of an independent set is at most the size of a spanning set. 
That is, if $I$ is a linearly independent set and $S$ is a spanning set, then
\[
|I| \le |S|.
\]
Thus, our strategy is to construct a suitable vector space $V$ together with:
\begin{itemize}
    \item a linearly independent set of vectors indexed by the functions in $\mathcal{H}$, and
    \item a spanning set whose size equals the right-hand side of the claimed inequality.
\end{itemize}
Since every spanning set must contain at least as many vectors as any independent set, this will immediately yield the desired upper bound on $|\mathcal{H}|$.

\paragraph{The Vector Space.}
We now describe the vector space $V$. 
Let $\mathbb{F}$ be any field containing at least $k$ distinct elements; for concreteness, we take $\mathbb{F}=\mathbb{R}$. Since $\mathcal{H} \subseteq [k]^n$, we may view $\mathcal{H}$ as a subset of $\mathbb{R}^n$. We consider
\[
V = \{ f : \mathcal{H} \to \mathbb{R} \},
\]
the space of all functions from $\mathcal{H}$ to $\mathbb{R}$, equipped with pointwise addition and scalar multiplication.
This space is naturally isomorphic to $\mathbb{R}^{|\mathcal{H}|}$, since each function is determined by its values on $\mathcal{H}$.

\paragraph{The Independent Set.}
For each $h \in \mathcal{H}$, define the indicator function
\[
\mathbf{1}_h(g) =
\begin{cases}
1 & \text{if } g = h, \\
0 & \text{otherwise}.
\end{cases}
\]
The collection $\{\mathbf{1}_h : h \in \mathcal{H}\}$ is linearly independent in $V$, and therefore forms an independent set of size $|\mathcal{H}|$.

\paragraph{The Spanning Set.}
We construct the spanning set using polynomials.
Since $\mathcal{H} \subseteq \mathbb{R}^n$, every polynomial 
\[
P(x_1,\ldots,x_n) \in \mathbb{R}[x_1,\ldots,x_n]
\]
defines a function in $V$ by restriction: namely, for each $h \in \mathcal{H}$ we evaluate $P(h)$ by plugging in the coordinates of $h$. 
Thus, each polynomial induces a function from $\mathcal{H}$ to $\mathbb{R}$.

Our spanning set will consist of the monomials
\[
x_1^{e_1} x_2^{e_2} \cdots x_n^{e_n},
\]
where the exponents satisfy:
\begin{itemize}
    \item $0 \le e_i < k$ for every $i$, and
    \item all but at most $d$ of the indices $i$ satisfy $e_i < \ell$.
\end{itemize}
The number of such monomials is
\[
\sum_{i=0}^{d} \binom{n}{i} (k-\ell)^{\,i} \ell^{\,n-i},
\]
which matches exactly the right-hand side of the inequality in the theorem. 
Observe that this set of monomials mirrors the structure of the extremal family from Example~1.

We now prove that the above set of monomials spans $V$. 
The proof proceeds by induction on $n$, while $k$, $\ell$, and $d$ are fixed.

\medskip
\noindent \textbf{Base case: $n=d$.}
When $n=d$, the condition that all but at most $d$ of the exponents satisfy~$e_i < \ell$ becomes vacuous. 
Thus, the spanning set consists of all monomials
\(x_1^{e_1} \cdots x_n^{e_n}\) with \(0 \le e_i < k.\)
We claim that these monomials span $V$. 
Indeed, since $\mathcal{H} \subseteq [k]^n$, it suffices to show that every function \(f : [k]^n \to \mathbb{R}\) can be represented by a polynomial $P(x_1,\ldots,x_n)$ in which each variable has degree strictly less than $k$.~\footnote{In the special case $k=2$, this reduces to the well-known fact from Boolean function analysis that every function on the Boolean cube $\{0,1\}^n$ has a multilinear polynomial representation.}
This follows from standard polynomial interpolation. 
In particular, for each $h \in [k]^n$, the indicator function $\mathbf{1}_h$ can be explicitly written as
\[
\prod_{i=1}^{n}
\left(
\prod_{\substack{j \in [k] \\ j \ne h(i)}}
\frac{x_i - j}{h(i) - j}
\right).
\]
Each factor depends only on a single coordinate and has degree at most $k-1$ in that coordinate. 
Therefore, every function on $[k]^n$ admits a polynomial representation in which each variable has degree less than $k$, and the claim follows.

\medskip
\noindent
\textbf{Induction step: $n>d$.}
Since $n>d$ and the $\ell$-DS dimension of $\mathcal{H}$ is $d$, the set $\mathcal{H}$ does not contain an $(\ell+1)$-pseudo-cube supported on all $n$ coordinates. Indeed, otherwise $[n]$ would be $\ell$-DS-shattered, contradicting the definition of $d$.
In particular, $\mathcal{H}$ is not an $(\ell+1)$-pseudo-cube. Hence, there exist $h_1\in\mathcal{H}$ and a coordinate $i\in[n]$ such that $h_1$ has fewer than~$\ell$ neighbors in direction $i$ within $\mathcal{H}$.

Let $\mathcal{H}_{-i}$ denote the projection of $\mathcal{H}$ onto the coordinates $[n]\setminus\{i\}$. 
By the induction hypothesis applied to $\mathcal{H}_{-i}$, there exists a polynomial $P_1$ (in the variables $x_j$ with $j\neq i$) satisfying the degree constraints and such that for every $h\in\mathcal H$,
\[
P_1(h)=
\begin{cases}
1 & \text{if } h(j)=h_1(j)\text{ for all } j\neq i,\\
0 & \text{otherwise}.
\end{cases}
\]
Let $S$ denote the set of values taken by the $i$-th coordinate among the neighbors of $h_1$ in $\mathcal{H}$.  Since $h_1$ has fewer than $\ell$ neighbors in direction $i$, we have $|S|<\ell$.
Define
\[
Q_1(x)
=
P_1(x)\prod_{j\in S}\frac{x_i - j}{h_1(i)-j}.
\]
Then $Q_1$ satisfies:
\[
Q_1(h_1)=1,
\qquad
Q_1(h)=0 \quad \text{for all } h \in \mathcal{H},\ h\neq h_1.
\]
Moreover, all individual degrees in $Q_1$ satisfy the required constraints: 
the degree in coordinate~$i$ is strictly less than $\ell$, and the remaining coordinates satisfy the induction hypothesis constraints.

We now remove $h_1$ from $\mathcal{H}$ and denote the remaining class by $\mathcal{H}_2$. 
Since $\mathcal{H}_2\subseteq \mathcal{H}$, it is also not an $(\ell+1)$-pseudo-cube, and the same argument produces $h_2 \in \mathcal{H}_2$ and a polynomial~$Q_2$ satisfying
\[
Q_2(h_2)=1,
\qquad
Q_2(h)=0 \quad \text{for all } h \in \mathcal{H}_2,\ h\neq h_2.
\]
Continuing like this, we obtain an ordering $h_1,h_2,\ldots$ of $\mathcal{H}$ and polynomials $Q_1,Q_2,\ldots$ such that
\[
Q_i(h_j)
=
\begin{cases}
1 & \text{if } i=j,\\
0 & \text{if } i>j.
\end{cases}
\]
Thus, the evaluation matrix $A_{ij}=Q_i(h_j)$ is triangular with ones on the diagonal. In particular, the polynomials $Q_1,\ldots,Q_{|\mathcal{H}|}$ are linearly independent in $V$. Since there are $|\mathcal{H}|$ of them and $\dim V = |\mathcal{H}|$, they form a basis of $V$.
Since each $Q_i$ is a linear combination of the allowed monomials, it follows that those monomials span $V$, as required.
\end{proof}

\subsection{Discussion and Open Problems}

\paragraph{On the algebraic nature of the proof.}

Our proof relies on a linear-algebraic argument. 
Although algebraic methods have repeatedly proved powerful in extremal combinatorics, it is natural to wonder whether the DS Sauer inequality admits a direct combinatorial proof.

Beyond aesthetic considerations, such a proof could shed light on the structural properties of hypothesis classes with finite DS dimension, and might even help address open problems in learning theory. A central open question in multiclass PAC learning is to determine the optimal dependence of the sample complexity on the DS dimension. 
The best known lower bound scales as $\mathrm{DS}/\varepsilon$, where $\varepsilon$ is the accuracy parameter, while the best known upper bound
(due to \cite{hanneke2024improved}) scales as
\[
\frac{\mathrm{DS}^{3/2} \log(\mathrm{DS}/\varepsilon)\log(\mathrm{DS})}{\varepsilon}.
\]
It is widely believed that the upper bound is not tight and can be improved to match the linear dependence suggested by the lower bound. In the binary setting, the optimal $\mathrm{VC}/\varepsilon$ upper bound is obtained via arguments that are closely tied to combinatorial proofs of the classical Sauer-Shelah-Perles lemma, such as the shifting/symmetrization argument and the inductive proof based on the inclusion graph~\cite{haussler1994predicting}.\footnote{There are alternative proofs of the optimal $\mathrm{VC}/\varepsilon$ bound, for example the argument of~\cite{hanneke16optimal}, which is effectively built upon uniform convergence analysis for certain conditional distributions. However, such uniform convergence techniques do not extend to the multiclass setting.}
It is, therefore, natural to speculate that a combinatorial proof of the DS Sauer inequality could lead to improved sample complexity bounds in the multiclass case as well.

\paragraph{Special cases admitting combinatorial proofs.}
In certain restricted settings, purely combinatorial arguments are available.

For example, when $n=2$, a class $\mathcal{H}\subseteq [k]\times[k]$ may be viewed as the edge set of a bipartite graph on $[k]\sqcup[k]$. 
In this case, having $\ell$-DS dimension $1$ is equivalent to the property that every nonempty subgraph contains a vertex of degree at most $\ell$. 

In this case, a direct combinatorial proof of the bound is obtained by repeatedly peeling off a vertex of degree at most $\ell$ and deleting all incident edges. 
Since each step removes at most $\ell$ edges and the process continues for $2k$ steps, a simple counting argument yields
\[
\lvert \mathcal{H}\rvert \le \ell(2k-\ell),
\]
which matches the bound of Theorem~\ref{thm:main} for $n=2$ and $d=1$.

There is also a combinatorial proof in the case $d=\ell=1$ and arbitrary $n$ and $k$. 
This argument was found by Nitay Amiel, an undergraduate student at the Technion, and is presented in Appendix~\ref{sec:nitay}.

\paragraph{Relationship between $\ell$-Natarajan and $\ell$-DS dimensions.}

It is also natural to ask how the $\ell$-Natarajan and $\ell$-DS dimensions compare.
It follows immediately from the definitions that
\[
\ell\text{-Nat}(\mathcal{H}) \le \ell\text{-DS}(\mathcal{H}).
\]
In the other direction, a general upper bound of the form
\[
\ell\text{-DS}(\mathcal{H})
\le
O\!\left(\ell \cdot \ell\text{-Nat}(\mathcal{H}) \cdot \log k\right)
\]
follows from the Sauer inequality for the $\ell$-Natarajan dimension.
Indeed, suppose a set of size~$n$ is $\ell$-DS-shattered. 
Then the restriction of $\mathcal{H}$ to this set contains an $(\ell+1)$-pseudo-cube, and in particular at least $\ell^n$ distinct vectors. 
On the other hand, if the $\ell$-Natarajan dimension equals $d$, the Natarajan-Sauer bound implies that the number of distinct vectors on this set is at most
\[
\ell^{\,n-d}
\sum_{i=0}^{d}
\binom{n}{i}
\binom{k}{\ell+1}^{\,i}.
\]
Comparing the two bounds and solving for $n$ yields
\[
n = O(\ell d \log k),
\]
as claimed.

It would be interesting to determine whether the factor of $\ell$ in this upper bound is necessary, or whether a sharper relationship between the two dimensions can be established.




\paragraph{Maximum classes for the DS dimension.}
In the classical VC setting, a class is called \emph{maximum} if it meets the Sauer--Shelah--Perles inequality with equality for all $n$. 
Such classes have been extensively studied in learning theory, combinatorics, geometry, and even computational biology~\cite{floyd:95,Kuzmin:07,Rubinstein:09,Rubinstein:12,rubinstein:15,moran2016labeled,alon2017sign,Chepoi:21,Chalopin:22,chase:24,lawrence:83,Bollobas:89,Gartner:94,Bollobas:95,Moran:12,chepoi:22,Dress:96,Bandelt:06}. 
They exhibit rich mathematical structure, with connections to geometry, topology, and algebra.

Since our Sauer-type inequality for the DS dimension is tight, DS-maximum classes exist as well. 
It is therefore natural to ask whether these classes possess additional structure, and whether they admit alternative characterizations beyond the extremal counting construction.

In the VC case, much of the structural understanding of maximum classes arises from combinatorial proofs of the classical Sauer lemma. 
Accordingly, exploring DS-maximum classes is naturally aligned with the search for a combinatorial proof of the DS Sauer inequality. 
A systematic study of DS-maximum classes therefore constitutes a promising direction for future work.

\paragraph{Other algebraic Sauer inequalities.}
The present result is not the only instance in which algebraic techniques have been used to derive Sauer-type inequalities. 
There exist linear-algebraic proofs of the classical Sauer inequality~\cite{FP83,Smo97,MR16}, as well as algebraic proofs for the corresponding inequalities for the Recursive Teaching dimension~\cite{Samei14teaching} and for the Graph dimension~\cite{Samei14compression}.\footnote{Related algebraic arguments have also been useful for other types of learning theoretic results.  For instance, \cite{pmlr-v19-yang11a} prove a result on identifiability of prior distributions over VC classes from data distributions, via a nested induction argument related to the one in our proof.} 
For the latter two dimensions, no purely combinatorial proofs are currently known. 
It would be interesting to understand whether such proofs exist, and whether they might reveal additional structural insights.

Moreover, in~\cite{Samei14compression}, the authors introduce a broader family of combinatorial dimensions - generalizing the Graph dimension - for which they establish tight Sauer-type bounds via algebraic methods. 
While the DS dimension does not appear to be a special case of this family, it would be interesting to clarify how it relates to these dimensions. 
For example, it is known that the DS dimension is upper bounded by the Graph dimension, but it is unclear how it compares to the other members of that family.

\section{Implications for List Learning}

List prediction is a natural and increasingly important extension of classical multiclass classification. 
Rather than outputting a single label, the learner outputs a short list of labels, and the prediction is deemed successful if the true label appears in this list.

This relaxation arises in many practical settings. 
In recommendation systems, for example, the goal is not to predict a single item but to present the user with a short list of plausible options, with success measured by whether the user selects one of them (see Figure~\ref{fig:list-recommendation}). 
Similarly, in large-scale classification problems with thousands of labels, it is often more meaningful to output a small candidate set than a single prediction. 
Another prominent example is the top-$k$ loss function, in which the model is allowed $k$ guesses for each sample. 
This loss is widely used in machine learning practice and competitions, and can be viewed as a variant of list learning; see, e.g.,~\cite{Lapin2015LossFF,YAN2018479}. 
List prediction also provides a principled way to address label ambiguity. 
In computer vision tasks, for instance, it may be inherently unclear whether an image depicts a pond or a river, and penalizing the model for every such mistake can be overly restrictive. 
Allowing a short list of plausible labels offers a natural alternative.

From a theoretical perspective, list learning has emerged as a central theme in multiclass learning with large label spaces. 
Allowing the learner to output a short list of candidate labels can substantially reduce sample complexity, and list prediction has played a key role in a series of recent advances in multiclass and structured learning 
\cite{brukhim2022characterization,
      charikar2023characterization,
      moran2023list,
      brukhim2023improperboosting,
      brukhim2024multiclass,
      hanneke2024list,
      hanneke2024improved,
      pabbaraju2024characterizationlistregression,
      hanneke2025a,
      hannekerepresentation,
      hanneke2025private,
      erez2025list,
      cohen2026natarajan,
      charikar2025limit}.

In some of these works, list learning is studied as a goal in its own right, while in others it serves as a structural tool for reducing the label space and for overcoming barriers that arise in classical multiclass PAC learning. List-valued prediction rules also appear in related frameworks such as \emph{conformal learning}, where algorithms output predictions together with measures of confidence or uncertainty. For example, such procedures may produce a set of plausible labels for an unlabeled test point $x$, which can be viewed as a list of candidate predictions. See the monograph of~\cite{vovk2005algorithmic} and the surveys~\cite{shafer2008tutorial,angelopoulos2021gentle} for further background. List learning is also closely connected to the setting of list-decodable learning in statistics; for a detailed exposition, see, for example, Chapter~5 of the recent textbook by~\cite*{diakonikolas2023algorithmic}.      

These considerations motivate the formal study of list PAC learning, in which the learner is permitted to output a list of size $L$, and performance is measured by whether the true label appears in the predicted list.

\begin{figure}[t]
\centering
\begin{tikzpicture}[
    >=Stealth,
    node distance=2.2cm,
    every node/.style={font=\small},
    box/.style={draw, rectangle, rounded corners, minimum height=1cm, minimum width=2.6cm},
    item/.style={draw, rectangle, minimum width=1.8cm, minimum height=0.6cm}
]

\node[box] (user) {User features $x$};

\node[box, right=of user] (predictor) {List predictor};

\node[right=of predictor, xshift=0.5cm] (listanchor) {};

\node[item] (i1) at ($(listanchor)+(0,0.8)$) {Movie $M_1$};
\node[item] (i2) at ($(listanchor)+(0,0)$) {Movie $M_7$};
\node[item] (i3) at ($(listanchor)+(0,-0.8)$) {$\checkmark$ Movie $M_{12}$};

\draw[->] (user) -- (predictor);
\draw[->] (predictor) -- (i2);

\end{tikzpicture}
\caption{Illustration of list prediction in a recommendation setting. 
Given user features $x$, the learner outputs a short list of items. 
Prediction succeeds if the user’s true preference appears in the list.}
\label{fig:list-recommendation}
\end{figure}
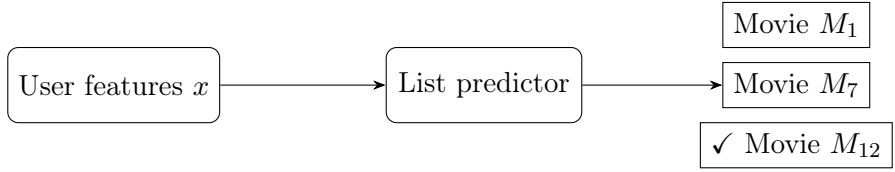

\subsection{List PAC Learning}

We now formalize the learning setting.
Let $\mathcal{X}$ be an instance space and $\mathcal{Y}$ a label set.  
A \emph{concept class} is a set $\mathcal{C}\subseteq \mathcal{Y}^{\mathcal{X}}$ of functions $c:\mathcal{X}\to \mathcal{Y}$.  
The learning goal is to identify an unknown target concept $c^\star \in \mathcal{C}$ from labeled examples.

In classical multiclass PAC learning, the learner outputs a hypothesis $h:\mathcal{X}\to\mathcal{Y}$ that assigns each input a single label.  
In $\ell$-list learning, the learner is allowed to output \emph{a set of at most $\ell$ labels} for each input.  
Formally, an $\ell$-list predictor is a function
\[
h : \mathcal{X} \to \binom{\mathcal{Y}}{\le \ell},
\]
where $\binom{\mathcal{Y}}{\le \ell}$ denotes the collection of subsets of $\mathcal{Y}$ of size at most $\ell$.  
Prediction is deemed correct on $(x,y)$ if $y \in \hat h(x)$.
We consider the standard realizable PAC model.  
For a distribution $\mathcal{D}$ over $\mathcal{X}\times \mathcal{Y}$ and a function
\(f : \mathcal{X} \to \binom{\mathcal{Y}}{\le \ell},\)
define the population loss
\[
\mathcal{L}_{\mathcal{D}}(f)
:=
\Pr_{(x,y)\sim\mathcal{D}}\!\big[y\notin f(x)\big].
\]
In the special case $\ell=1$, we identify $\binom{\mathcal{Y}}{\le 1}$ with $\mathcal{Y}$, and the condition
$y\notin f(x)$ reduces to the usual misclassification event $f(x)\neq y$.
A distribution $\mathcal{D}$ is said to be realizable by $\mathcal{C}$ if
\[
\inf_{c\in\mathcal{C}} \mathcal{L}_{\mathcal{D}}(c) = 0.
\]
A learning algorithm receives $m$ i.i.d.\ samples from $\mathcal{D}$ and outputs an $\ell$-list predictor $\hat h$.

\begin{definition}[$\ell$-list PAC learnability and sample complexity]
Fix $\ell\in\mathbb{N}$. 
A concept class~$\mathcal{C}$ is \emph{$\ell$-list PAC learnable} if there exists a learning algorithm $\mathcal{A}$ such that for every choice of error parameter $\varepsilon\in(0,1)$ and confidence parameter $\delta\in(0,1)$ there exists an integer 
\(m_{\mathcal{C}}(\varepsilon,\delta)\)
with the following property: for every distribution $\mathcal{D}$ realizable by $\mathcal{C}$, when given $m_{\mathcal{C}}(\varepsilon,\delta)$ i.i.d.\ samples from $\mathcal{D}$, the algorithm outputs an $\ell$-list predictor $\hat h$ satisfying
\[
\Pr\big[\,\mathcal{L}_{\mathcal{D}}(\hat h) \le \varepsilon\,\big]
\;\ge\; 1-\delta.
\]
The minimal such function $m_{\mathcal{C}}(\varepsilon,\delta)$ is called the 
\emph{sample complexity} of $\ell$-list PAC learning for~$\mathcal{C}$.
\end{definition}
Thus, the sample complexity quantifies the number of examples required to guarantee,
with confidence at least $1-\delta$, that the resulting $\ell$-list predictor
has error at most $\varepsilon$.

\medskip

Classical multiclass PAC learning corresponds to the special case $\ell=1$.
Allowing for larger lists weakens the prediction requirement and can strictly reduce sample complexity, particularly when the label space is large.

The combinatorial parameter governing $\ell$-list PAC learnability is the $\ell$-DS dimension.
For the remainder of this section, we write $\DS_\ell(\mathcal{C})$ to denote the $\ell$-DS dimension of a concept class $\mathcal{C}$. 
While the preceding sections focused on classes \(\mathcal{H}\subseteq [k]^n\), the definition of the \(\ell\)-DS dimension naturally extends to arbitrary concept classes \(\mathcal{C}\subseteq \mathcal{Y}^{\mathcal{X}}\).
Formally, \(\DS_\ell(\mathcal{C})\) is the largest integer \(d\) for which there exist points \(x_1,\dots,x_d\in\mathcal{X}\) such that the set
\[
\{(c(x_1),\dots,c(x_d)) : c\in\mathcal{C}\}
\]
contains a \emph{finite} \(\ell\)-pseudo-cube.
If such sets exist for arbitrarily large \(d\), we define \(\DS_\ell(\mathcal{C})=\infty\).

A fundamental result in the theory of list learning states that a concept class \(\mathcal{C}\) is \(\ell\)-list PAC learnable if and only if 
\(\DS_\ell(\mathcal{C})<\infty\)~\cite{charikar2023characterization,brukhim2022characterization}. 
This characterization was established by~\cite{brukhim2022characterization} in the case \(\ell=1\) (the classical single-label setting), and was later extended by~\cite{charikar2023characterization} to arbitrary \(\ell\).
Beyond this qualitative characterization, a basic and central quantitative question is how the sample complexity depends on \(d := \DS_\ell(\mathcal{C})\) and on $\ell$.
The analysis of~\cite{charikar2023characterization} yields an upper bound scaling as
\[
\widetilde{O}\!\left(\frac{\ell^{6}d^{3/2}}{\varepsilon}\right),
\]
up to the standard logarithmic dependence on $\delta$. 
Subsequently,~\cite{brukhim2024multiclass} developed a boosting-based approach, obtaining an incomparable bound of order
\[
\widetilde{O}\!\left(\frac{\ell^{4}d^{5}}{\varepsilon}\right).
\]
Using the sharp Sauer inequality for the $\ell$-DS dimension established in this work, we obtain the following improved upper bound.
\begin{theorem}[PAC upper bound for $\ell$-list learning]\label{thm:list-pac-upper}
Let $\mathcal{C}\subseteq \mathcal{Y}^{\mathcal{X}}$ be a concept class with 
$d = \DS_\ell(\mathcal{C}) < \infty$. 
Then $\mathcal{C}$ is $\ell$-list PAC learnable with sample complexity
\[
m_{\mathcal{C}}(\varepsilon,\delta)
=
\widetilde{O}\!\left(
\frac{\ell d^{3/2} + \ell d \log(1/\delta)}{\varepsilon}
\right),
\]
where $\widetilde{O}(\cdot)$ hides polylogarithmic factors in $d$, $\ell$, and $1/\varepsilon$.
\end{theorem}
Notably, the bound has no dependence on the size of the label space, and in particular applies also when \(\mathcal{Y}\) is infinite.

\subsection{List Uniform Convergence}
\label{subsec: list uniform convergence}
Uniform convergence and empirical risk minimization (ERM), defined below, are among the central principles of statistical learning theory.
In the binary setting, the Fundamental Theorem of PAC Learning states that PAC learnability, uniform convergence, and the ERM principle are equivalent~\cite{VC71}; see also the textbook treatment in~\cite{shalev2014understanding}. 
More generally, for finite label spaces, uniform convergence and PAC learnability are known to coincide~\cite{shalev-shwartz10uniform}. 
Beyond its theoretical significance, ERM is also closely connected to widely used statistical principles such as maximum likelihood estimation.

Informally, uniform convergence refers to the phenomenon whereby, given sufficiently many samples drawn from a distribution, the empirical losses of all hypotheses in a class simultaneously approximate their true population losses. 
This principle underlies the validity of ERM: if empirical losses uniformly approximate true losses, then minimizing empirical loss is a sound strategy for approximately minimizing population loss.

In the context of list learning, following~\cite{hanneke2024list}, we study uniform convergence for hypothesis classes of \(\ell\)-list predictors rather than single-label predictors. 
This is natural, since a list learning algorithm outputs an \(\ell\)-list predictor and therefore selects its hypothesis from such a class. 
Accordingly, the notion of the \(\ell\)-DS dimension naturally extends to classes of list predictors. Let \(\listC\) be a class of \(\ell\)-list predictors \(c:\mathcal{X}\to \binom{\mathcal{Y}}{\le \ell}\). 
We say that a set \(\{i_1,\ldots,i_d\} \subseteq [n]\) is \emph{\(\ell\)-DS-shattered} by \(\listC\) if
\[
\{(y_1,y_2,\ldots,y_d): \exists c \in \listC,\ \forall j,\ y_j \in c(i_j)\}
\]
contains an \((\ell+1)\)-pseudo-cube. 
The \emph{\(\ell\)-DS dimension} of \(\listC\) is the maximum size of a set that is \(\ell\)-DS-shattered.
For a dataset $S=((x_i,y_i))_{i=1}^n$, define the empirical loss
\[
\mathcal{L}_{S}(c)
:=
\frac{1}{n}\sum_{i=1}^n \mathbf{1}\{y_i \notin c(x_i)\}.
\]

\begin{definition}[Uniform convergence and uniform sample complexity]
A hypothesis class~$\listC$ of $\ell$-list predictors satisfies \emph{uniform convergence} if there exists a function
\[
m_{\listC}^{\mathrm{UC}} : (0,1)^2 \to \mathbb{N}
\]
such that for every error parameter $\varepsilon\in(0,1)$ and confidence parameter $\delta\in(0,1)$ the following holds: for every distribution $\mathcal{D}$ over $\mathcal{X}\times\mathcal{Y}$, if $S$ consists of $m_{\listC}^{\mathrm{UC}}(\varepsilon,\delta)$ i.i.d.\ samples from $\mathcal{D}$, then with probability at least $1-\delta$,
\[
\sup_{h\in\listC}
\big|
\mathcal{L}_{\mathcal{D}}(c)
-
\mathcal{L}_{S}(c)
\big|
\le \varepsilon.
\]
The minimal such function $m_{\listC}^{\mathrm{UC}}(\varepsilon,\delta)$ is called the 
\emph{uniform convergence sample complexity} of~$\listC$.
\end{definition}
It was shown by~\cite{hanneke2024list} that any $\ell$-list learnable class 
$\listC\subseteq \mathcal{Y}^{\mathcal{X}}$ with $|\mathcal{Y}|=k$ and 
$\DS_\ell(\listC) = d$ satisfies uniform convergence with sample complexity
\[
m_{\listC}^{\mathrm{UC}}(\varepsilon,\delta)
=
\widetilde{O}\!\left(
\frac{\ell^2 d + \log(1/\delta)}{\varepsilon^2}
\right),
\]
where $\widetilde{O}(\cdot)$ hides logarithmic factors in $k$, $d$, $\ell$.

The following theorem exploits the sharp Sauer inequality for the $\ell$-DS dimension established in this work, yielding an improvement in the dependence on $\ell$.

\begin{theorem}[Uniform convergence for $\ell$-list predictors]\label{thm:UC-upper}
Let $\listC\subseteq \mathcal{Y}^{\mathcal{X}}$ be a class of $\ell$-list predictors with 
$d = \DS_\ell(\listC) < \infty$ and $|\mathcal{Y}|=k$. 
Then $\listC$ satisfies uniform convergence with sample complexity
\[
m_{\listC}^{\mathrm{UC}}(\varepsilon,\delta)
=
\widetilde{O}\!\left(
\frac{\ell d + \log(1/\delta)}{\varepsilon^2}
\right),
\]
where $\widetilde{O}(\cdot)$ hides polylogarithmic factors in $d$, $\ell$, $k$.
\end{theorem}




\subsection{Overview of the Proofs}

In this subsection, we will give the proof sketches for learning and uniform convergence bounds stated above for $\ell$-list prediction, and the formal proofs can be found in \cref{sec:formal-proof-list-pac} and \cref{sec:formal-proof-list-uc}.

Our arguments are based on the proof techniques developed in the works of~\cite{charikar2023characterization,brukhim2024multiclass,hanneke2024list}, whose bounds we refine. 
In particular, we follow their overall reduction from sample complexity bounds to combinatorial growth bounds.
For the list learning guarantees, we refine several parts of the argument, in addition to substituting our new sharp Sauer lemma.
For completeness and clarity, we present the proofs in a largely self-contained manner, highlighting the points at which the new sharp Sauer inequality for the $\ell$-DS dimension yields improved quantitative bounds.

\subsubsection{Proof sketch for \cref{thm:list-pac-upper}}


At a high level, the proof proceeds in three stages. We first learn an intermediate $\ell'$-list predictor for some $\ell' \gg \ell$. This predictor's role is to retain the true label with high probability while substantially narrowing the set of candidate labels (from possibly infinite to finite). We then use this intermediate predictor to filter a new independent sample, thereby reducing the original problem to a list learning problem with a finite label space. Finally, we solve this problem to obtain the target $\ell$-list predictor.

Let $d=\DS_{\ell}(\listC)$. In the first stage, we apply the learner of \cite{charikar2023characterization} to an i.i.d. sample $S_1$ of size $m_1$, obtaining an intermediate $\ell'$-list predictor $\mu^{\ell'}$, where $\ell' \le \ell (ed)^{\sqrt d}\log(2m_1).$

In the second stage, we draw another independent sample $S_2$ of size $m_2$ and retain only those examples whose labels belong to the predicted list: $S_{\mathrm{real}}=\{(x,y)\in S_2:\; y\in \mu^{\ell'}(x)\}.$
Since the distribution is realizable by $\listC$, the filtered sample $S_{\mathrm{real}}$ is still realizable by $\listC$. Thus, the first two stages reduce the original list PAC learning problem to learning from a realizable sample, with labels from a $\ell's$-list. 

The key tool in the third stage is the one-inclusion-list algorithm of \cite{charikar2023characterization}. At a high level, this algorithm is analyzed via the one-inclusion graph of the restriction of $\listC$ to the sample. The vertices of this graph correspond to candidate labelings, and two vertices are adjacent if they differ on a single sample point. An orientation of this graph induces a prediction rule, and the sample complexity is controlled by the maximum $\ell$-outdegree of the chosen orientation. Thus, the learning problem is reduced to a combinatorial question: how small can one make the relevant outdegree in the one-inclusion graph?

This is precisely where our main improvement enters. A central quantity in the analysis is the $\ell$-exponential dimension $\mathrm{E}_{\ell}(\listC)$ from \cite{charikar2023characterization}, which controls the maximum $\ell$-outdegree. The previous argument upper bounded this quantity in terms of $\ell$-Natarajan dimension using a Sauer-type bound whose dependence on $\ell$ is quadratic \cite{daniely2015inapproximability,charikar2023characterization}. Our optimal Sauer lemma sharpens this counting step by saving a factor of $\ell$.

Concretely, let $g=\mathrm{E}_{\ell}(\listC)$, and let $S\in\cX^g$ be a sample of size $g$ that is $\ell$-exponential shattered by $\listC$. Writing $\cH=\listC|_S$, the definition of $\ell$-exponential dimension gives $|\cH|\ge (\ell+1)^g$. On the other hand, our optimal Sauer lemma implies the upper bound
\[
|\cH|\le \ell^g \left(\frac{e\ell' g}{d}\right)^d.
\]
Comparing the lower and upper bounds yields
\[
g=O(\ell d\log \ell').
\]
Thus, our Sauer lemma gives a sharper control of $\mathrm{E}_{\ell}(\listC)$ than in \cite{charikar2023characterization}.

We combine this with two further improvements in the third stage. First, previous work \cite{charikar2023characterization} bounds the maximum $\ell$-outdegree by the average degree of the one-inclusion graph, and then upper bounds that average degree by $\ell^2 \mathrm{E}_{\ell}(\listC)$ via induction. Our argument instead uses a sharper flow-based orientation method, inspired by \cite{haussler1994predicting}, to control the maximum $\ell$-outdegree by the shifting average degree of the graph. We then show that this shifting average degree is at most $\mathrm{E}_{\ell}(\listC)$. This yields an additional factor-$\ell^2$ improvement. The underlying construction is implemented via an auxiliary integer flow network, and the details are deferred to the formal proof.
Second, we use a standard confidence amplification argument from \cite{haussler1994predicting} to obtain the final hypothesis, which saves another factor of $\ell^2$. More specifically, we repeat the one-inclusion list learner on independent sub-samples and aggregate the resulting hypotheses by a standard voting step. Together, these refinements improve the sample complexity upper bound by a factor of $\ell^5$ compared to \cite{charikar2023characterization}.

Choosing $m_1=\widetilde{O}\!\left(\frac{d^{3/2}+\log(1/\delta)}{\epsilon}\right)
\qquad$ and $\qquad
m_2=\widetilde{O}\!\left(\frac{\ell d\log(\ell'/\delta)}{\epsilon}\right),$
we obtain the total sample complexity $m=m_1+m_2
=
\widetilde{O}\!\left(
\frac{\ell d^{3/2}+\ell d\log(1/\delta)}{\epsilon}
\right).$
The formal algorithms and proofs are given in \cref{sec:formal-proof-list-pac}.

\subsubsection{Proof sketch for \cref{thm:UC-upper}}
Before providing the proof sketch of the theorem, we first introduce the graph dimension \cite{vapnik1974theory, natarajan1989learning} that is known to characterize uniform convergence. A sequence $S=\{x_i\}_{i=1}^n$ is $\Gl$-shattered by a $\ell$-list concept class $\listC$ if there is  $p\in \cY^n$ called a pivot, such that for any $b\in \{0,1\}^n$ there is $c_b\in \listC$ such that for all $i \in [n]$, we have: $\mathbf{1}\{p_i\in c_b(x_i)\} = b_i$. 
The $\ell$-graph dimension of $\listC$ is $\Gl(\listC)$ the size of the largest $\Gl$-shattered sequence, or infinity if there are $\Gl$ shattered sequences of arbitrary size.

Now, according to \cite{vapnik1974theory}, the VC dimension fully characterizes the uniform convergence of a binary hypothesis class. In particular, note that the $\ell$-graph dimension of a list hypothesis class can be viewed as the VC dimension of the class of binary functions $\{(x,y) \mapsto \mathbf{1}\{y \notin c(x)\} : c \in \listC\}$. In this way, for a list hypothesis class $\listC$, the $\ell$-graph dimension $\Gl(\listC)$ fully characterizes uniform convergence and the rate would be $\Theta\left(\sqrt{\frac{\Gl(\listC)}{n}}\right)$, where $n$ is the sample size. Thus, the remainder of the proof is upper bounding the $\ell$-graph dimension by $\ell$-DS dimension.

The proof largely follows that of \cite{hanneke2024list}, and we use our optimal Sauer lemma to obtain a sharper bound. The proof proceeds by lower bounding the cardinality of the projection class induced by $\listC$ on a sample of size $\Gl(\listC)$ that is $\ell$-graph shattered by $\listC$, and upper bounding the same quantity in terms of $\DS_\ell(\listC)$.

Let $d = \DS_\ell(\listC)$ and $ g = \Gl(\listC)$. Let $S\in \cX^g$ be a sample of size $g$ that is $\ell$-graph shattered by $\listC$. Define $\cH=\cH(\listC\vert_S)=\{h\in \cY^g\;:\; \exists c\in\listC ,\; h \prec c|_{S}\}$ as the class of projections induced by $\listC$ on the sample $S$. The lower bound follows from \cite{hanneke2024list}, which gives $|\cH| \geq \frac{(2\ell)^g}{4(2\ell-1)^g}\ell ^g$.
Next, the upper bound follows from our optimal Sauer lemma. In particular, by \cref{thm:main}, we have $|\cH| \leq \ell ^ g (gk)^{\DS_{\ell}(\cH)}$. Moreover, $\DS_{\ell}(\listC)\geq \DS_{\ell}(\cH)$, since any sample that is $\ell$-DS-shattered by $\cH$ is also $\ell$-DS-shattered by $\listC$. Thus, we have the upper bound $|\cH| \leq \ell ^ g (gk)^{d}$. Combining the lower and upper bounds, we have $d\geq \widetilde{\Omega}\left(\frac{g}{\ell(\log g + \log k)}\right)$. Finally, combining this inequality with the uniform convergence rate in terms of the $\ell$-graph dimension completes the proof.
\subsection{Complete proofs of \cref{thm:list-pac-upper} and \cref{thm:UC-upper}}
In this section, we provide the complete proofs for \cref{thm:list-pac-upper} and \cref{thm:UC-upper}.

\subsubsection{Proof of \cref{thm:list-pac-upper}}\label{sec:formal-proof-list-pac}
We first introduce the $\ell$-exponential dimension \cite{charikar2023characterization} (see also \cite{brukhim2022characterization}), a dimension directly defined in terms of the number of patterns realizable by the hypothesis class $\listC$. 
\begin{definition}[$\ell$-exponential dimension \cite{charikar2023characterization}]
    We say that a hypothesis class $\listC\subseteq \cY^{\cX}$ $\ell$-exponential shatters a sequence $S\in \cX^{d}$ if $|\cH|\geq(\ell+1)^{d}$ where $\cH = \listC|_S$ is the restriction of $\listC$ to $S$. The $\ell$-exponential dimension of a hypothesis class $\listC$, denoted by $\mathrm{E}_{\ell}(\listC)$, is the largest integer $d$ such that $\listC$ $\ell$-exponential shatters some sequence~$S\in\cX^{d}$.
\end{definition}
For simplicity, we denote $\ell$-DS dimension $\DS_{\ell}(\listC)$ as $\dds$ and $\ell$-exponential dimension $\mathrm{E}_{\ell}(\listC)$ as $\dexp$. Next, we define an important combinatorial object associated with a hypothesis class: the one-inclusion graph.

\begin{definition}[One-inclusion graph \cite{haussler1994predicting, rubinstein2006shifting}]
The one-inclusion graph of $\cH \subseteq \cY^m$ is a hypergraph $G(\cH)=(V,E)$ that is defined as follows. The vertex-set is $V=\cH$. For each $i \in [m]$ and $f:[m]\setminus\{i\} \to\cY$, let $e_{i,f}$ be the set of all $h \in \cH$ that agree with $f$ on $[m]\setminus \{i\}$. The edge-set is
    \begin{equation*}
        E = \{e_{i,f}: i \in [m], f:[m] \setminus \{i\} \to \cY, e_{i,f}\neq \emptyset \}.
    \end{equation*}
    We say that the edge $e_{i,f}\in E$ is in the direction $i$, and is adjacent to the hypothesis/vertex $h$ if $h \in e_{i,f}$. Every vertex $h\in V$ is adjacent to exactly $m$ edges. The size of the edge $e_{i,f}$ is the size of the set $|e_{i,f}|$.
\end{definition}
With respect to the one-inclusion graph, we consider the degrees of its vertices and the average degree of the graph.
\begin{definition}[$\ell$-degree \cite{charikar2023characterization}]
    \label{def: L-degree}
    Let $G(\cH)=(V,E)$ be the one-inclusion graph of $\cH \subseteq \cY^m$. The ${\ell}$-degree of a vertex $v \in V$ is
    \begin{align*}
        \degree(v) &= |\{e \in E: v \in e, |e|>\ell\}|.
    \end{align*}
\end{definition}

\begin{definition}[Average ${\ell}$-degree \cite{charikar2023characterization}]
    \label{def: avg-L-degree}
    Let $G(\cH)=(V,E)$ be the one-inclusion graph of $\cH \in \cY^m$. The average ${\ell}$-degree of $\cH$ is
    \begin{align*}
        \avd(\cH) = \frac{1}{|V|}\sum_{v \in V}\degree(v) = \frac{1}{|V|}\sum_{e \in E:|e|>{\ell}}|e|.
    \end{align*}
\end{definition}
Next, we restate the definition of list orientation from \cite{charikar2023characterization}. The notion of orienting refers to assigning directions to the edges of a one-inclusion graph toward the lists of vertices they are adjacent to. Such an orientation captures the behavior of a deterministic learning algorithm when it predicts the label of an unlabeled test point, given a set of labeled examples as input.

\begin{definition}[List Orientation\cite{charikar2023characterization}]
    A list orientation $\sigmal$ of the one-inclusion graph $G(\cH) = (V,E)$ having list size $L$ is a mapping $\sigmal:E \to \{V'\subseteq V: |V'|\le \ell\}$ such that for each edge $e \in E$, $\sigmal(e) \subseteq e$.
    
\end{definition}
With respect to a list orientation $\sigma_\ell$, one can consider the $\ell$-outdegree of this orientation. In particular, we have the following definition.
\begin{definition}[${\ell}$-outdegree of a list orientation \cite{charikar2023characterization}]
    Let $G(\cH)=(V,E)$ be the one-inclusion graph of $\cH$, and let $\sigmal$ be a ${\ell}$-list orientation of it. The ${\ell}$-outdegree of $v \in V$ in $\sigmal$ is
    \begin{equation*}
        \outdeg(v;\sigmal) = |\{e:v \in e, v \notin \sigmal(e)\}|.
    \end{equation*}
    The maximum ${\ell}$-outdegree of $\sigmal$ is
    \begin{equation*}
        \outdeg(\sigmal) = \sup_{v \in V} \outdeg(v;\sigmal).
    \end{equation*}
\end{definition}



Now, we restate the definition of list realizability from \cite{charikar2023characterization}. Moreover, for completeness, we include the corresponding learning algorithm.
\begin{definition}[List realizability \cite{brukhim2022characterization}]
A sample $S \in (\cX \times \cY)^m$ is \emph{realizable} by the list $\mu$ if $y\in \mu(x)$ for every $(x,y)$ in $S$. A distribution $\cD$ over $\cX \times \cY$ is {\em realizable} by $\mu$ if for every $m \in \mathbb{N}$, a random sample $S\sim \cD^m$ is realizable by $\mu$ with probability~$1$.
\end{definition}

\begin{algorithm}[H]
\caption{One-inclusion list algorithm $\cA_{\listC,\mu^{\ell'}}^{\ell}$ for a class $\listC\subseteq\cY^{\cX}$ and list $\mu^{\ell'}$} 
\label{algo: learn a L-list from L'-list}
\begin{flushleft}
  {\bf Input:} A sample $S = \big((x_1, y_1),\cdots,(x_m, y_m)\big)$ realizable by $\cH$ and $\mu^{{\ell}'}$.\\
{\bf Output:} A ${\ell}$-list hypothesis $\cA_{\listC,\mu^{{\ell}'}}^{{\ell}}(S)=\mu^{\ell}_S:\cX \to \{Y \subseteq \cY: |Y|\le {\ell}\}$.\\
\ \\
For each $x \in \cX$, the ${\ell}$-list $\mu^{{\ell}}_S(x)$ is computed as follows.
\end{flushleft}
\begin{algorithmic}[1]
\State Consider the class $\cH \subseteq \cY^{m+1}$ of all patterns over the \emph{unlabeled data} that are realizable by both $\listC$ and $\mu^{{\ell}'}$ i.e., $\cH=\{h \in \listC|_{(x_1,\ldots,x_m,x)} : h(m+1) \in \mu^{{\ell}'}(x) \text{ and } h(i) \in \mu^{{\ell}'}(x_i)\}$.
\State Find a ${\ell}$-list orientation $\sigma^{{\ell}}$ of $G(\cH)$ that \textit{minimizes} the \textit{maximum} ${\ell}$-outdegree.
\State Consider the edge in direction $m+1$ defined by $S$:
\[e =\{ h \in \cH:  \forall i \in [m]  \ \ 
h(i) =y_i\}.\]
\State Set $\mu^{{\ell}}_S(x) = \{h(m+1):h \in \sigma^{{\ell}}(e)\}$.
\end{algorithmic}
\end{algorithm}

Our proof adapts the conventional technique for bounding the leave-one-out error of a one-inclusion graph. The proof of this part can be divided into two steps: (1) In the first step, we upper bound the maximum ${\ell}$-outdegree by ${\ell}$-exponential dimension (see \cref{lemma: list orientation}). Compared to Corollary 6.5 in \cite{charikar2023characterization}, we save a factor of ${\ell}^2$. (2) In the second step, using \cref{thm:main}, we can further increase the bounding of the ${\ell}$-exponential dimension by $O({\ell}\dds)$, (see \cref{corollary: bound exp by L LDS}). Compared to Lemma 6.6 in \cite{charikar2023characterization}, we save another factor of ${\ell}$. 
Combining step 1 and step 2, and comparing with the analysis in Proposition 7.3 of \cite{charikar2023characterization}, we improve the bound by a factor of ${\ell}^3$ in total.

Then, we start our analysis of \cref{algo: learn a L-list from L'-list}, and start by bounding the maximum outdegree of the one-inclusion list algorithm. 
In previous papers \cite{haussler1995sphere,rubinstein2006shifting,brukhim2022characterization,charikar2023characterization}, shifting is a major tool for bounding outdegree. For a beautiful introduction on the shifting operation, we refer the reader to the beginning of Section 3 in \cite{brukhim2022characterization}. For completeness, we restate the definition and some major properties that we will use in our proof as follows.
\begin{definition}[Shifting \cite{brukhim2022characterization}]
    \label{def:shifting}
    Let $\cH \subseteq [k]^m$ and let $i \in [m]$. The shifting operator in the $i^{\text{th}}$ direction $\shift_i$ maps $\cH$ to its shifted version $\shift_i(\cH)$ as follows. Shifting is first defined on edges. For $f : [m] \setminus \{i\} \to [k]$, let $e_{f}$ be the collection of $h \in \cH$ that agree with $f$ on $[m] \setminus \{i\}$. The shifting $\shift_i(e_f)$ is obtained by ``pushing
    $e_f$ downward''; namely, $\shift_i(e_f)$ is the collection of all $g \in [k]^m$ that agree with $f$ on $[m] \setminus \{i\}$ and $1 \leq g(i) \leq |e_f|$. The shifting of $\cH$ is the union of all shifted edges
    $$\shift_i(\cH) = \bigcup_{f} \shift_i(e_f) \subseteq [k]^m.$$
\end{definition}
Shifting has the useful property that the fixed point $\cH_*$ of the shifting operations (i.e., $\mathbb{S}_i(\cH_*) = \cH_*$ for all $i \in [m]$) is guaranteed to exist and is furthermore closed downward (i.e., if $h \in \cH_* \subseteq [k]^m$, and $g_i \le h_i \, \forall i \in [m]$, then $g \in \cH^*$).

We restate the following result from \cite{charikar2023characterization}: the shifting operation does not increase the ${\ell}$-exponential dimension. This is an immediate corollary of Claim 22 in \cite{brukhim2022characterization}.

\begin{lemma}[Corollary 6.2, \cite{charikar2023characterization}]
    \label{lemma: exponential dimension shifting}
    For every $\cH \subseteq [k]^m$ and $i \in [m]$,
    \begin{align*}
        \dexp(\mathbb{S}_i(\cH)) \le \dexp(\cH).
    \end{align*}
    Thus, $\dexp(\cH_*)\le \dexp(\cH)$, where $\cH_*$ is the fixed point of the shifting operations i.e., $\mathbb{S}_i(\cH_*) = \cH_*$ for all $i \in [m]$.
\end{lemma}

Next, we define the shifting average ${\ell}$-degree.
\begin{definition}[Shifting average ${\ell}$-degree \cite{charikar2023characterization}]
    Let $G(\cH)=(V,E)$ be the one-inclusion graph of $\cH \subseteq [k]^m$. The shifting average ${\ell}$-degree of $\cH$ is
    \begin{align*}
        \savd(\cH) = \frac{1}{|V|}\sum_{e \in E}(|e|-{\ell})_+, 
    \end{align*}
    where $(x)_+ = \max(x,0)$.
\end{definition}

Let $G(\cH)=(V,E)$ be the one-inclusion graph of $\cH \subseteq [k]^m$. Given a subset of vertices $U\subseteq V$, we define the induced one-inclusion graph $G(\cH)[U]$ to be the hypergraph with vertex set $U$ and edge set $E' = \{e\cap E: e \in E, |e\cap U|\ge 1\}$. There is another equal definition of an induced one-inclusion graph defined by the subset $\cH' \subseteq \cH$ as $G\left(\cH'\right)$. We now define the maximum density of $G(\cH)$ using the second induced one-inclusion graph as follows.
\begin{definition}[Maximum ${\ell}$-density of $G(\cH)$]\label{def: L density of G(H)}Let $G(\cH)=(V,E)$ be the one-inclusion graph of $\cH \subseteq [k]^m$. The maximum ${\ell}$-density of $\cH$ is
\begin{align*}
    \mathrm{MD}^{{\ell}}(\cH) = \max_{\cH'\subseteq \cH}\savd\left(G\left(\cH'\right)\right).
\end{align*}
\end{definition}

Subsequently, we restate the result from \cite{charikar2023characterization} that the shifting operation does not decrease the shifting average ${\ell}$-degree $\savd$. 
\begin{lemma}[Lemma 6.3, \cite{charikar2023characterization}]
    \label{lemma: savd shifting}
    For every $\cH \subseteq [k]^m$ and $i \in [m]$,
    \begin{align*}
        \savd(\shift_i(\cH)) \ge \savd(\cH).
    \end{align*}
    Thus, $\savd(\cH_*) \ge \savd(\cH).$
\end{lemma}


To control the max ${\ell}$-outdegree in the one-inclusion graph, after proving these results, \cite{charikar2023characterization} tends to prove an upper bound on the average degree of the one-inclusion graph $G(\cH)$. Instead of that, we prove an upper bound on the shifting average degree $\savd(\cH)$, which diverges from theirs and leads to a sharper (near optimal) upper bound. Now, we start to prove the upper bound of the shifting average ${\ell}$-degree of $G(\cH)$ by ${\ell}$-exponential dimension.  
\begin{lemma}[Shifting average ${\ell}$-degree is bounded by ${\ell}$-exponential dimension]\label{lemma: svad leq exponential}
For every $\cH\in[k]^{m}$, we have
\begin{equation*}
    \savd(\mathcal{H}) \leq \dexp(\mathcal{H}).
\end{equation*}
\end{lemma}
\begin{proof}[Proof of \Cref{lemma: svad leq exponential}]
According to \Cref{lemma: exponential dimension shifting}, we have that $\dexp(\mathcal{H_*}) \leq \dexp(\mathcal{H})$. According to \Cref{lemma: savd shifting}, we have that $\savd(\mathcal{H_*}) \geq \savd(\mathcal{H})$. We only need to prove $\savd(\mathcal{H_*}) \leq \dexp(\mathcal{H_*})$.
We prove this by induction on the size of $\mathcal{H}_*$. For the base case, when $|\mathcal{H}_*| = 1$, we have $\savd(\mathcal{H_*}) = \dexp(\mathcal{H_*}) = 0$, and so the inequality holds. The induction step is justified as follows. Let $h_0$ be the concept in $\mathcal{H}_*$ so that no concept in $\mathcal{H}_*$ is larger than $h_0$ with respect to the natural partial order. Let $| h _ 0 |$ be the number of coordinates of $h_0$ that are larger than ${\ell}$. Because $\mathcal{H_*}$ is closed downwards by the definition of shifting, all concepts $h \leq h_0$ are contained in $\mathcal{H}_*$, which implies that $\text{deg}^{{\ell}}(h_0) = |h_0|$. This also implies that $\mathcal{H}_*$ restrict to those $|h_0|$ coordinates has size at least $({\ell} + 1)^{|h_0|}$. Combining together, we have $\text{deg}^{{\ell}}(h_0) = |h_0|\leq \dexp(\mathcal{H}_*) $. Now, consider the class $\mathcal{H}^{\prime}_{*} = \mathcal{H}_{*} \setminus \{h_0\}$, which is also closed downwards, and the inductive hypothesis implies that $\savd(\mathcal{H}^{\prime}_{*}) \leq \dexp(\mathcal{H}^{\prime}_{*})\leq \dexp(\mathcal{H}_{*})$. Now consider those edges in $\mathcal{H}_*$ that contain $h_0$ according to their sizes. For simplicity, we denote $e^{\prime}\in G\left(\mathcal{H}^{\prime}_{*}\right)$ as the corresponding edge $e\in G\left(\mathcal{H}_{*}\right)$. 1.If such $e$'s size is less than or equal to ${\ell}$, then the value $(|e| - {\ell})_+ = 0$ in $G\left(\mathcal{H}_{*}\right)$, for the corresponding edge $e ^\prime$, whose value $(|e^{\prime}| - {\ell})_+ = 0$ also holds. 2. If $e$'s size is greater than ${\ell}$, then the number of such edges is $\text{deg}^{{\ell}}(h_0)$. The value of those edges satisfies $(|e^{\prime}| - {\ell})_+ = (|e| - {\ell})_+ - 1$. Now, we have
\begin{align*}
    \savd(\mathcal{H}_{*}) & = \frac{1}{m} \Big(\sum_{e\in E, h_0\notin e}(|e| - {\ell})_+  + \sum_{e\in E, h_0\in e, |e| > {\ell}}(|e| - {\ell})\Big)\\
    & \allowdisplaybreaks = \frac{1}{m} \Big(\sum_{e^\prime \in E^{\prime}}(|e^\prime| - \ell)_+ + \text{deg}^{{\ell}}(h_0)\Big)\\
    & = \frac{1}{m} \Big((n-1)\savd(\mathcal{H}_{*}^{\prime}) + \text{deg}^{{\ell}}(h_0)\Big)\\
    &\leq \frac{1}{m} \Big((n-1)\dexp(\mathcal{H}^{\prime}_{*}) + \dexp(\mathcal{H}_{*})\Big)\\
    &\leq \dexp(\mathcal{H}_{*}).
\end{align*}
\end{proof}

This immediately leads to the following corollary.

\begin{corollary}[${\ell}$-exponential dimension bound maximum $L$-density of $G(\cH)$]\label{lemma: MD leq exponential} For every $\mathcal{H}\subseteq[k]^{m}$ with $\ole(\cH) = \dexp(\cH)$, and every class $\mathcal{H^{\prime}}\subset\mathcal{H}$, we have
\begin{equation*}
    \savd(\mathcal{H}^{\prime}) \leq \dexp(\mathcal{H}).
\end{equation*}
This implies that $\mdl(\mathcal{H}) \leq \dexp(\mathcal{H})$.
\end{corollary}
\begin{proof}[Proof of \Cref{lemma: MD leq exponential}]
    According to \Cref{lemma: svad leq exponential}, we have $\savd(\mathcal{H}^{\prime}) \leq \dexp(\mathcal{H}^{\prime})$. We also have the naive inequality $\dexp(\mathcal{H}^{\prime}) \leq \dexp(\mathcal{H})$. Combining the two inequalities above, we obtain our final result.
\end{proof}

From here, we follow the proof technique of \cite{haussler1994predicting}, orienting the one-inclusion graph via constructing an integer flow network, and extend their result in binary classification to the multiclass setting. For $\cH \subseteq [k]^{m}$ and one-inclusion graph $G(\cH) =(V,E)$, we construct an integer flow network $I(\cH)$ (as shown in \cref{fig:network-flow}) with a source $s$, a sink $t$, and two intermediate node layers derived from the one-inclusion graph $G(\cH)$. The first layer contains a node $e_i$ for each hyperedge $e_i$ in $G(\cH)$, while the second layer contains a node $v_j$ for each vertex $v_j$. The network's directed edges and their respective capacities are defined as follows:
\begin{enumerate}
    \item An edge connects the source $s$ to each first-layer node $e_i$ with a capacity of $(|e_i| - {\ell})_{+}$.
    \item For each hyperedge $e_i$, an edge with capacity $1$ connects its corresponding node $e_i$ to each second-layer node $v_j$ for which the vertex $v_j \in e_i$.
    \item An edge connects each second-layer node $v_j$ to the sink $t$ with a capacity of $\ceil[\big]{\mdl(\mathcal{H})}$.
\end{enumerate}
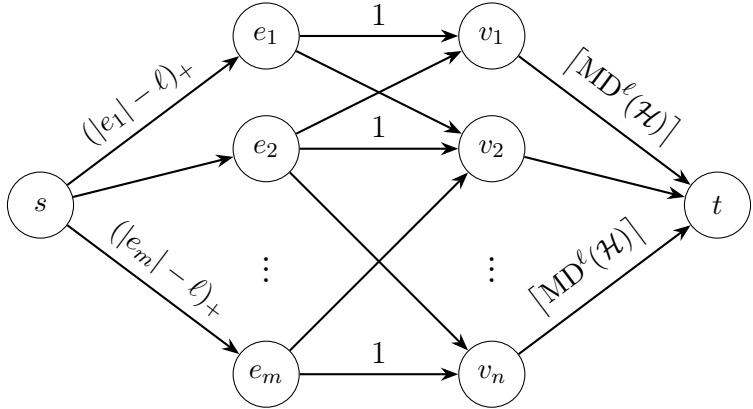
\begin{figure}[H]
\begin{center}
\begin{tikzpicture}[
    mycircle/.style={
        circle,
        draw=black,
        fill opacity = 0.3,
        text opacity=1,
        inner sep=0pt,
        minimum size=25pt, 
        font=\small
    },
    myarrow/.style={
        -Stealth,
        thick
    },
    mydots/.style={
        font=\Large
    },
    node distance=1.5cm and 2.5cm
]
    \def\vdist{1.5cm}

    \node[mycircle] (s) at (0, 0) {$s$};

    \coordinate (L2-center) at (3,0);
    \coordinate (L3-center) at (6,0);

    \node[mycircle] (e1) at ([yshift=1.5*\vdist]L2-center) {$e_1$};
    \node[mycircle] (e2) at ([yshift=0.5*\vdist]L2-center) {$e_2$};
    \node[mydots] (dots-e) at ([yshift=-0.5*\vdist]L2-center) {\vdots};
    \node[mycircle] (em) at ([yshift=-1.5*\vdist]L2-center) {$e_m$};

    \node[mycircle] (v1) at ([yshift=1.5*\vdist]L3-center) {$v_1$};
    \node[mycircle] (v2) at ([yshift=0.5*\vdist]L3-center) {$v_2$};
    \node[mydots] (dots-v) at ([yshift=-0.5*\vdist]L3-center) {\vdots};
    \node[mycircle] (vn) at ([yshift=-1.5*\vdist]L3-center) {$v_n$};

    \node[mycircle] (t) at (9, 0) {$t$};

    \foreach \i/\weight in {e1/(|e_1| - {\ell})_{+}, e2/, em/(|e_m| - {\ell})_{+}}
        \draw [myarrow] (s) -- node[sloped, above, font=\small] {$\weight$} (\i);

    \draw [myarrow] (e1) -- node[above] {1} (v1);
    \draw [myarrow] (e1) -- (v2);
    
    \draw [myarrow] (e2) -- (v1);
    \draw [myarrow] (e2) -- node[above] {1} (v2);
    \draw [myarrow] (e2) -- (vn);

    \draw [myarrow] (em) -- (v2);
    \draw [myarrow] (em) -- node[above] {1} (vn);
    \foreach \i/\weight in {v1/\ceil[\big]{\mdl(\mathcal{H})}, v2/, vn/\ceil[\big]{\mdl(\mathcal{H})}}
        \draw [myarrow] (\i) -- node[sloped, above, font=\small] {$\weight$} (t);
\end{tikzpicture}
\caption{integer flow network $I(\cH)$}
\label{fig:network-flow}
\end{center}
\end{figure}
We first show that for the above network $I(\cH)$, there is an integer flow to achieve the maximum flow.
\begin{lemma}\label{lemma: maxflow-mincut}
    For the network $I(\cH)$, there is an integer flow achieving the maximum flow, $\sum_{i = 1}^{m} (|e_i| - {\ell})_{+} = |V| \cdot \savd(\cH)$.
\end{lemma}
\begin{proof}[Proof of \Cref{lemma: maxflow-mincut}]
The proof relies on the \textit{Max-Flow Min-Cut Theorem}. First, we will establish that the maximum flow value is precisely $\sum_{i = 1}^{m} (|e_i| - {\ell})_{+}$. Then, the existence of an integer-valued flow achieving this maximum follows directly from the \textit{Integrality Theorem} of \cite{Ford_Fulkerson_1956}, as all capacities in the network are integers. The total capacity of the edges leaving the source $s$ is $\sum_{i=1}^{m} (|e_i| - {\ell})_{+}$. By definition, the maximum flow cannot exceed this value. We will now show that the capacity of any $s-t$ cut is at least this value, thereby proving the equality. Consider an arbitrary $s-t$ cut $(S, T)$, where $s \in S$ and $t \in T$. The capacity of the cut, $C(S,T)$, is the sum of capacities of all edges $(u,v)$ with $u \in S$ and $v \in T$. Let's partition the nodes as $E_S, E_T, V_S, V_T$, where
\begin{itemize}
    \item $E_S = \{e_i \mid e_i \in S\}$ and $E_T = \{e_i \mid e_i \in T\}$
    \item $V_S = \{v_j \mid v_j \in S\}$ and $V_T = \{v_j \mid v_j \in T\}$
\end{itemize}

The cut capacity is the sum of three components:
$$C(S,T) = \sum_{e_i \in E_T} \underbrace{\text{cap}(s, e_i)}_{\text{Type 1}} + \sum_{v_j \in V_S} \underbrace{\text{cap}(v_j, t)}_{\text{Type 2}} + \sum_{\substack{e_i \in E_S, \\ v_j \in e_i \cap V_T}} \underbrace{\text{cap}(e_i, v_j)}_{\text{Type 3}}$$
Substituting the capacities gives:
$$C(S,T) = \sum_{e_i \in E_T} (|e_i| - {\ell})_{+} + |V_S| \cdot \ceil[\big]{\mdl(\mathcal{H})} + \sum_{e_i \in E_S} |e_i \cap V_T|$$
Our goal is to show that $C(S,T) \geq \sum_{i=1}^{m} (|e_i| - {\ell})_{+}$. Since the sum over $E_T$ is already part of the target, we only need to prove:
$$|V_S| \cdot \ceil[\big]{\mdl(\mathcal{H})} + \sum_{e_i \in E_S} |e_i \cap V_T| \geq \sum_{e_i \in E_S} (|e_i| - {\ell})_{+}$$
Rearranging the terms, this is equivalent to showing:
$$|V_S| \cdot \ceil[\big]{\mdl(\mathcal{H})} \geq \sum_{e_i \in E_S} \left( (|e_i| - {\ell})_{+} - |e_i \cap V_T| \right)$$
We can bound the term in the parentheses using the fact that for any $e_i \in E_S$, $|e_i| = |e_i \cap V_S| + |e_i \cap V_T|$. This gives the key inequality:
$$ (|e_i| - {\ell})_{+} - |e_i \cap V_T| \leq (|e_i \cap V_S| - {\ell})_{+} $$
Applying this inequality, it is sufficient to prove the following:
$$|V_S| \cdot \ceil[\big]{\mdl(\mathcal{H})} \geq |V_S| \cdot \frac{1}{|V_S|}\sum_{e_i \in E_S} (|e_i \cap V_S| - {\ell})_{+}$$
This inequality holds by the definition of $\ceil[\big]{\mdl(\mathcal{H})}$ in \cref{def: L density of G(H)}, which is the maximum shifting average ${\ell}$-degree on any subgraph of the original one inclusion graph $G(\cH)$, and must be greater or equal to the shifting average ${\ell}$-degree on the subgraph induced by $V_S$. Since the capacity of any cut is at least $\sum_{i=1}^{m} (|e_i| - {\ell})_{+}$, the minimum cut capacity must also be at least this value. By the Max-Flow Min-Cut Theorem, the maximum flow is equal to the minimum cut capacity, which concludes the proof.
\end{proof}

Next, we show that we can orient the one-inclusion graph $G(\cH)$ leveraging the integer flow in $I(\cH)$, such that the maximum ${\ell}$-outdegree can be upper bounded by the ${\ell}$-exponential dimension $\dexp(\cH)$.

\begin{lemma}[${\ell}$-Exponential Dimension Leads to List Orientation]\label{lemma: list orientation}
    For every $\mathcal{H}\in[k]^{m}$, there is a ${\ell}$-list orientation $\sigmal$ of $G(\cH)$ with maximum ${\ell}$-outdegree at most $\ceil [\big]{\mdl(\mathcal{H})} \leq \dexp(\mathcal{H})$. 
\end{lemma}

\begin{proof}[Proof of \Cref{lemma: list orientation}]
We construct the orientation $\sigma^{{\ell}}$ using the integer-valued maximum flow, $f$, guaranteed by the proof of \Cref{lemma: maxflow-mincut}. Recall that in the flow network $I_{\mathcal{H}}(\bar{x})$, edges from hyperedge nodes $\{e_i\}$ to vertex nodes $\{v_j\}$ have a capacity of 1, meaning the flow $f(e_i, v_j)$ across any such edge is either 0 or 1.
We define the orientation $\sigma^{{\ell}}$ of the one-inclusion graph $G(\cH)$ based on this flow. For each hyperedge $e$, we orient it to the vertices $\{v^{\prime}\}$ such that $f(e,v^{\prime}) = 0$ in the flow $f$.
Now, we show that this orientation respects the claimed outdegree bound. By our rule, an edge is directed away from $v_j$ (to some $e_i$) only if the flow $f(e_i, v_j)$ is 1. Therefore, the ${\ell}$-outdegree of any vertex $v_j$ is the sum of flows on all incident edges from the hyperedge layer:
$$\outdeg(v_j) = \sum_{e_i \ni v_j} f(e_i, v_j)$$
By the principle of \textit{flow conservation} at node $v_j$ in the network $I(\cH)$, the total flow entering $v_j$ from the hyperedge nodes must equal the total flow leaving $v_j$ to the sink $t$, we have
$$\sum_{e_i \ni v_j} f(e_i, v_j) = f(v_j, t)$$
Finally, any flow is bounded by the capacity of the edge it traverses. The capacity of the edge $(v_j, t)$ is $\ceil[\big]{\mdl(\mathcal{H})}$. Thus, we have:
$$f(v_j, t) \le \text{cap}(v_j, t) = \ceil[\big]{\mdl(\mathcal{H})}$$
Combining these steps and \Cref{lemma: MD leq exponential}, we get the upper bound on the outdegree of any vertex $v_j$ as $\ceil[\big]{\mdl(\mathcal{H})} \leq \dexp(\mathcal{H})$.

\end{proof}

Subsequently, we will show how to upper bound the ${\ell}$-exponential dimension by the ${\ell}$-DS dimension. 

\begin{corollary}[Controlling the ${\ell}$-exponential dimension by ${\ell}$-DS dimension] \label{corollary: bound exp by L LDS}For every $\cH\subseteq[k]^{m}$ with $\olds(\cH) = \dds$ and $\ole(\cH) = \dexp < \infty$, we have\footnote{In this paper, $\log(\cdot)$ denotes the natural logarithm, and we define $\log(x) := \max\{\log(x), 1\}$.}
\[\dexp \le  40\ell\dds\log k.\]
\end{corollary}
\begin{proof}[Proof of \cref{corollary: bound exp by L LDS}]
    Consider the set $S$ which realizes $\dexp$, so that $|S| = \dexp$. Then we have that
    \[|\cH|_{S}|\ge({\ell}+1)^{\dexp}.\]
    Also, from \cref{thm:main}, we have
    \[|\cH|_{S}|\le {\ell}^{\dexp}\left(\frac{ek\dexp}{\dds}\right)^{\dds}.\]
    Combining the above two inequalities, we obtain
    \[({\ell}+1)^{\dexp}\le {\ell}^{\dexp}\left(\frac{ek\dexp}{\dds}\right)^{\dds}.\]
    By rearranging the term, we have
    \[\left(\frac{{\ell}+1}{{\ell}} \right)^{\dexp} \le \left(\frac{ek\dexp}{\dds}\right)^{\dds} .\]
    Taking the logarithm of both sides and rearranging the terms, we get
    \[\frac{\dexp}{\dds} \le \frac{\log\left(\frac{ek\dexp}{\dds}\right)}{\log\left(1+ \frac{1}{{\ell}}\right)}.\]
    Using the basic inequality $\log(1+\frac{1}{{\ell}})>\frac{1}{2{\ell}}$ when $\ell>1$, we have
    \[\frac{\dexp}{\dds} \le 2{\ell} \log\left(ek\right) + 2{\ell} \log\left(\frac{\dexp}{\dds}\right) .\]
    Leveraging Lemma A.2. in \cite{shalev2014understanding}, we have
    \[\frac{\dexp}{\dds}\le 4{\ell} \log\left(ek\right) + 8{\ell} \log\left(4{\ell}\right) .\]
    Thus, we have:
    \begin{align*}
        \dexp & \le \left(4{\ell} \log\left(ek\right) + 8{\ell} \log\left(4{\ell}\right)\right) \dds \\
        & \le \ell\dds\left(8\log k + 16\log 2 + 4\right)\\
        &\le 40 {\ell}\dds\log k.
    \end{align*}
    
\end{proof}
 
Armed with \cref{lemma: list orientation} and \cref{corollary: bound exp by L LDS}, we are ready to bound the leave-one-out error of \cref{algo: learn a L-list from L'-list}.
\begin{lemma}
    \label{lemma: loo learn L from L'}
    Let $\listC \subseteq \cY^\cX$ be a concept class with $\olds(\listC) = \dds <\infty$ and let $\mu^{{\ell}'}$ be a ${\ell}'$-list (where ${\ell}' \gg \ell$). For every distribution $\cD$ over $\cX \times \cY$ that is realizable by both $\listC$ and by~$\mu^{{\ell}'}$, and for all integers $m>0$,
    \begin{align*}
        \Pr_{(S,(x,y)) \sim \cD^{m+1}}\left[y\not\in\mu^{\ell}_S(x) \right] \leq \frac{40\ell \dds \log({\ell}')}{m},
    \end{align*}
    where $\mu^k_S = \cA_{\listC,\mu^{{\ell}'}}^{{\ell}}(S)$.
\end{lemma}

\begin{proof}
    Let $\cD$ be a distribution that is realizable by $\listC$ and $\mu^{{\ell}'}$. By the leave-one-out symmetrization argument,
    \begin{align*}
        \Pr_{(S, (x,y))\sim \cD^{m+1} }\left[y \not \in \mu^{\ell}_{S}(x)\right] 
        = \Pr_{(S',i) \sim \cD^{m+1} \times \mathrm{Unif}(m+1)}\left[y'_ i \not \in \mu^{\ell}_{S'_{-i}}(x'_i)\right],
    \end{align*}
    where $\mu^k_S = \cA_{\listC,\mu^{{\ell}'}}^{{\ell}}(S)$. It therefore suffices to show that for every sample $S'$ that is realizable by $\listC$ and $\mu^{{\ell}'}$,
    \begin{equation*}
        \Pr_{i \sim \mathrm{Unif}(m+1)}\left[y'_i \not\in \mu^{\ell}_{S'_{-i}}(x'_i) \right] \le \frac{40{\ell} \dds \log({\ell}')}{m}.
    \end{equation*}
    Fix $S'$ which is realizable by $\listC$ and $\mu^{{\ell}'}$ for the rest of the proof. The class {$\cH = \listC|_{(x'_1,\cdots,x'_{m+1})}$} constructed by the algorithm $\cA_{\listC,\mu^{{\ell}'}}^{{\ell}}$ for $S'_{-i}$ and $x'_i$ is the same for all values of~$i$, and is realizable by $\mu^{{\ell}'}$. The ${\ell}$-DS dimension of $\cH$ is at most that of $\listC$. Denote by $\sigma^{{\ell}}$ the orientation of $G(\cH)$ that the algorithm chooses. \Cref{lemma: list orientation} tells us that the maximum ${\ell}$-outdegree of $\sigma^{{\ell}}$ is at most $d_{E}^{{\ell}}$, which can be further upper bounded by $40\ell \dds \log({\ell}')$ leveraging \cref{corollary: bound exp by L LDS}. Let~$y'$ denote the vertex in $G(\cH)$ defined by $y'=(y'_1,\cdots,y'_{m+1})$, and let $e_i$ denote the edge in the $i^{\text{th}}$ direction adjacent to $y'$. Then, we have:
    \begin{align*}
        \Pr_{i \sim \mathrm{Unif}(m+1)}\left[y'_i\not\in\mu^{\ell}_{S'_{-i}}(x'_i)\right] &= \frac{1}{m+1}\sum_{i=1}^{m+1}\mathds{1}\left[y'_i\not\in\mu^{\ell}_{S'_{-i}}(x'_i)\right] \\
        & \allowdisplaybreaks = \frac{1}{m+1}\sum_{i=1}^{m+1}\mathds{1}\left[y'\not\in\sigma^{{\ell}}(e_i) \right] \\
        &= \frac{\outdeg(y';\sigma^{{\ell}})}{m+1} \\
        &\le \frac{40{\ell} \dds \log({\ell}')}{m+1}.
    \end{align*}
\end{proof}
The rest of the proof adapts the two-stage framework of \cite{charikar2023characterization, brukhim2022characterization}. The stages are as follows: \textbf{(1)} We learn a $\ell'$-list hypothesis with list size ${\ell}'\gg {\ell}$. \textbf{(2)} We then use this $\ell'$-list hypothesis to filter for realizable samples, re-apply the one-inclusion algorithm, and use the confidence amplification technique of \cite{haussler1994predicting} to distill the final $\ell$-list. 

For the first stage, we restate the results from \cite{charikar2023characterization} without proof. For a detailed algorithm and proof, one can refer to \cite{charikar2023characterization}. 
\begin{lemma}[Learn a $\ell'$-list hypothesis \cite{charikar2023characterization}]\label{lemma: PAC learning L'}
    Let $\listC \subseteq \cY^\cX$ be a concept class with $\olds(\listC) = \dds <\infty$. Then for every distribution $\cD$ on $(\cX\times\cY)$ realizable by $\listC$, every $\epsilon,\delta\in(0,1)$, with probability at least $1-\delta$ over a sample of $S$ of size $m = O\left(\frac{(\dds)^{1.5}\log((\dds)^{1.5}/\epsilon)^2+\log(1/\delta)}{\epsilon}\right)$ drawn $i.i.d$ from $\cD$, we can learn a ${\ell}'$-list hypothesis $\mu^{{\ell}'}$ with error rate at most $\epsilon$,
    where ${\ell}' \le {\ell} (e\dds)^{\sqrt{\dds}} \log(2m).$
\end{lemma}
In the second step of our proof, we show that any concept class $\listC$ with a finite ${\ell}$-DS dimension is learnable given realizability with respect to a large list $\mu^{{\ell}'}$ (where ${\ell}' \gg {\ell}$). Our procedure, formalized in \cref{algo: produce a pac list learner}, utilizes a confidence amplification technique from \cite{haussler1994predicting}. This analysis, which we detail here, is the source of our improvement, saving another factor of $\ell^2$ compared to the analysis in Lemma 7.5 of \cite{charikar2023characterization}. Our algorithm operates in two stages, mirroring the classic "generate-and-select" amplification strategy. First, we construct a set of candidate hypotheses $\cG$. This is done by running \cref{algo: learn a L-list from L'-list} for $\ceil*{\log(2/\delta)}$ times. Each $i$-th run $h_i = \cA_{\listC,\mu^{{\ell}'}}^{{\ell}}(S_i)$ is trained on an independent sample $S_i$ with size $O({\ell}\dds\log({\ell}')/\epsilon)$ realizable by $\listC$ and $\mu^{{\ell}'}$ of size $O({\ell}\dds\log({\ell}')/\epsilon)$, and the resulting hypothesis $h_i$ is added to $\cG$. Second, we perform an ERM-based selection. We draw a new, independent validation sample $S_{val}$ of size $O(\log(1/\delta)/\epsilon)$. The final output $\cA(S)$ is the hypothesis in $\cG$ that achieves the lowest empirical error on $S_{val}$.
\begin{algorithm}[H]
\caption{List PAC learner for a class $\listC\subseteq\cY^{\cX}$ and list $\mu^{{\ell}'}$} 
\label{algo: produce a pac list learner}
\begin{flushleft}
  {\bf Input:} A sample $S = \big((x_1, y_1),\cdots,(x_m, y_m)\big)$ realizable by $\listC$ and $\mu^{{\ell}'}$.\\
{\bf Output:} A $\ell$-list hypothesis $\cA_{\mathrm{PAC}}^{{\ell}}(S)=\mu^{\ell}_S:\cX \to \{Y \subseteq \cY: |Y|\le {\ell}\}$.\\
\end{flushleft}
\begin{algorithmic}[1]
\State Initialize $\cG = \emptyset$. Partition $S$ into $\{S_1, S_2, \cdots S_{\ceil*{\log(2/\delta)}}, S_{val}\}$.
\For {$i = 1, 2,\cdots, \ceil*{\log(2/\delta)}$}
    \State Run \cref{algo: learn a L-list from L'-list} with input sample $S_i$ and get the list predictor $\mu^{{\ell}}_{S_{i}}$. 
    \State Update $\cG = \cG\cup \{\mu^{{\ell}}_{S_{i}}\}$
\EndFor
\State Output the hypothesis in $\cG$ with lowest empirical error on test set $S_{val}$.
\end{algorithmic}
\end{algorithm}
The following lemma shows that \cref{algo: produce a pac list learner} is a $\ell$-list PAC learner given a sample realizable by $\cH$ and a list $\mu^{{\ell}'}$ with ${\ell}' \gg {\ell}$.
\begin{lemma}\label{lemma: PAC learn L given L'}
    Let $\listC \subseteq \cY^\cX$ be a concept class with $\olds(\listC) = \dds<\infty$ and let $\mu^{{\ell}'}$ be a ${\ell}'$-list. Then for every distribution $\cD$ on $(\cX\times\cY)$ realizable by $\listC$ and $\mu^{{\ell}'}$, every $\epsilon,\delta\in(0,1)$, with probability at least $1-\delta$ over a sample of $S$ of size $m = O\left(\frac{\ell\dds\log({\ell}'/\delta)}{\epsilon}\right)$ drawn $i.i.d$ from $\cD$, the ${\ell}$-list hypothesis $\mu^{{\ell}} = \cA^{{\ell}}_{\mathrm{PAC}}(S)$ output by \cref{algo: produce a pac list learner} satisfies
    \[\Pr_{(x,y)\sim\cD}\left[y\not\in\mu^{{\ell}}(x)\right]\le\epsilon.\]
\end{lemma}
\begin{proof}
Let $p = \ceil*{\log(2/\delta)}$. The algorithm partitions $S$ into $\{S_i\}_{i=1}^p$ and $S_{val}$, with $|S_i| = n = \frac{160 \ell\dds\log({\ell}')}{\epsilon}$ and $|S_{val}| = m_{val} = \frac{32\log(2/\delta)}{\epsilon}+\log(p+1)$. The proof proceeds in two stages, with a final union bound.
   In the first stage of generation of the candidate hypothesis in $\cG$, according to \cref{lemma: loo learn L from L'}, the expected error rate for each $\mu_{S_i}^{{\ell}}$ is at most $\epsilon / 4$. Thus, by Markov inequality, each hypothesis $\mu_{S_i}^{{\ell}}$ in $\cG$ has error larger than $\epsilon/2$ with probability smaller than $1/2$. Since a different sample is used for each hypothesis $\mu_{S_i}^{{\ell}}$ in $\cG$. It follows that $\cG$ has a hypothesis $\mu^{{\ell}}$ of error at most $\epsilon / 2 $ with probability at least $1-\delta/2$. In the ERM selection step, according to lemma 5.1. in \cite{haussler1994predicting}, if $\cG$ has a hypothesis of error at most $\epsilon /2 $ then the ERM step in \cref{algo: produce a pac list learner} executed with a sample $S_{val}$ described above, outputs with probability at least $1-\delta/2$ a hypothesis in $\cG$ with error at most $\epsilon$.
\end{proof}

Subsequently, armed with \cref{lemma: PAC learning L'} and \cref{lemma: PAC learn L given L'}, we are now ready to prove \cref{thm:list-pac-upper}. 
\begin{proof}[Proof of \cref{thm:list-pac-upper}]
The proof is a three-stage argument. We construct the final $\ell$-list learner by first learning an intermediate ${\ell}'$-list, then using it to filter a new sample set to be "realizable," and finally learning our target $\ell$-list from this filtered set. We use a union bound over the failure probability of each stage. First, by \cref{lemma: PAC learning L'}, we draw an i.i.d. sample $S_1$ of size
\[m_1 = O\left(\frac{(\dds)^{1.5}\log((\dds)^{1.5}/\epsilon)^2+\log(1/\delta)}{\epsilon}\right).\]
This yields an ${\ell'}$-list hypothesis $\mu^{{\ell'}}$ with an error rate $\text{err}(\mu^{{\ell'}}) \le \epsilon$ with probability at least $1-\delta/4$. The resulting list size ${\ell'}$ is bounded by ${\ell'} \le {\ell} (e\dds)^{\sqrt{\dds}} \log(2m_1)$.
Second, we draw a new, independent sample $S_2$ of size
\[m_2 = O\left(\frac{\ell\dds\log({\ell'}/\delta)}{\epsilon}+\log\left(\frac{1}{\delta}\right)\right).\]
We filter this sample to keep only points realizable by $\mu^{{\ell'}}$, creating the set $S_{real} = \{(x,y) \in S_2 \mid y \in \mu^{{\ell'}}(x)\}$. A standard Chernoff bound shows that with probability at least $1-\delta/4$, this realizable sub-sample $S_{real}$ has a size of at least
\[m_2' = O\left(\frac{{\ell}\dds\log({\ell'}/\delta)}{\epsilon}\right).\]
We assume the data distribution $\cD$ is realizable by $\cH$, so $S_{real}$ is also realizable by $\cH$.
Third, we use this realizable sample $S_{real}$ (of size $m_2'$) as input to \cref{algo: produce a pac list learner}. By \cref{lemma: PAC learn L given L'}, this algorithm outputs a final $\ell$-list hypothesis $\mu^{\ell}$ with $\text{err}(\mu^{\ell}) \le \epsilon$ with probability at least $1-\delta/4$.
By a union bound over these three stages, the total failure probability is at most $\delta/4 + \delta/4 + \delta/4 = 3\delta/4 \le \delta$.
The total sample complexity is $m = m_1 + m_2$. We substitute the bound for ${\ell'}$ from Stage 1 into the $\log({\ell'})$ term of $m_2$. This gives the total sample size:
\begin{align*}
    m = O&\left( \frac{(\dds)^{1.5}\log((\dds)^{1.5}/\epsilon)^2+\ell\dds\log(1/\delta)}{\epsilon} +\frac{1}{\epsilon}\ell\dds\log\left(\frac{1}{\delta}\right)\right.\\
    &\left.\frac{1}{\epsilon}\ell\dds\log\left((e\dds)^{\sqrt{\dds} }\log\left(\frac{(\dds)^{1.5}\log((\dds)^{1.5}/\epsilon)^2+\log(1/\delta)}{\epsilon}\right)\right)\right).
\end{align*}
Omitting the polylogarithmic in $\dds$ and $1/\epsilon$, 
we have the dominant term in the complexity is
\begin{align*}
    &\widetilde{O}\left( \frac{(\dds)^{1.5}}{\epsilon}+ \frac{{\ell}\dds \cdot \log(1/\delta)}{\epsilon}+ \frac{{\ell}\dds \cdot \log({\ell'})}{\epsilon} \right) \\
    &= \widetilde{O}\left( \frac{(\dds)^{1.5}}{\epsilon} + \frac{{\ell}\dds \cdot \log(1/\delta)}{\epsilon}+ \frac{{\ell}\dds \cdot \sqrt{\dds}}{\epsilon} \right) \\
    &= \widetilde{O}\left(\frac{{\ell}(\dds)^{1.5}+{\ell}\dds\log(1/\delta)}{\epsilon}\right).
\end{align*}
Thus, the overall sample complexity is $\widetilde{O}\left(\frac{{\ell}(\dds)^{1.5}+{\ell}\dds\log(1/\delta)}{\epsilon}\right)$.
\end{proof}

\subsubsection{Proof of \cref{thm:UC-upper}}
\label{sec:formal-proof-list-uc}
\begin{proof}[Proof of \cref{thm:UC-upper}]
    Let $d = \DS_\ell(\listC)$ and $ g = \Gl(\listC)$. Let $S\in \cX^g$ be a sample of size $g$ that is $\ell$-graph shattered by $\listC$. Define $\cH=\cH(\listC\vert_S)=\{h\in \cY^g\;:\; \exists c\in\listC ,\; h\prec c|_{S}\}$ as the class of all possible patterns induced by $\listC$ on the sample $S$. 
    We first restate the following lemma from \cite{hanneke2024list} that lower bounds the cardinality of $\cH$.
    \begin{lemma}[\cite{hanneke2024list}]
        \label{lemma: lower bound f}
        Given an $\ell$-list class $\listC\subseteq\binom{\cY}{\leq\ell}^{\cX}$ that $\ell$-graph shatters $S\in \cX ^{n}$, we let $p\in\cY^n$ be a pivot and let $\{c_b\}_{b\in\{0,1\}^n}\subseteq \listC|_{S}$ be witnesses for the shattering such that $p_i\in c_{b}(x_i)$ if and only if $b_i = 1$. Denote by $A_b = \{y \in \cY ^ n : \forall i\; y_i \in c_{b}(x_i)\}$, the set of functions realizable by $c_b$. Then we have
        \[|\cup_{b\in\{0,1\}^n}A_b|\geq \frac{(2\ell)^n}{4(2\ell-1)^n}\ell^n.\]
    \end{lemma}
Observe that $\cup_{b\in\{0,1\}^g}A_b\subseteq\cH$, since every function in $A_b$ is realized by some projection of $\listC$ on $S$. Combining with $|S| = g$, \cref{lemma: lower bound f} yields
\[|\cH| \geq  |\cup_{b\in\{0,1\}^g}A_b|\geq\frac{(2\ell)^g}{4(2\ell-1)^g}\ell ^g.\]
We next upper bound the cardinality of $\cH$ using our optimal Sauer lemma. 
According to \cref{thm:main}, we have \[|\cH| \leq \ell ^ g (gk)^{\DS_{\ell}(\cH)}.\] We also note that $\DS_{\ell}(\listC)\geq \DS_{\ell}(\cH)$, since any sample that is $\ell$-DS-shattered by $\cH$ is also $\ell$-DS-shattered by $\listC$. Thus, we have the upper bound \[|\cH| \leq \ell ^ g (gk)^{d}.\] Combining the lower and upper bounds on $\cH$, we have
\[\ell ^ g (gk)^{d}\geq\frac{(2\ell)^g}{4(2\ell-1)^g}\ell ^g.\]
Taking logarithms on both sides of the inequality, we obtain
\[d(\log g + \log k)\geq g\log\left(1+\frac{1}{2\ell-1}\right)-2\log2.\]
Leveraging the basic inequality $\log(1+x)\ge \frac{x}{2}\; \forall\; x\in[0,1]$, we have
\[d \ge \frac{g}{4\ell(\log g + \log k)}-\frac{2\log2}{\log g + \log k},\] which yields
\begin{equation}
\label{eq:upper bound g by d}
    d = {\Omega}\left(\frac{g}{\ell(\log g + \log k)}\right).
\end{equation}
For the sample complexity of uniform convergence, we first restate the following result from \cite{vapnik1974theory,daniely2011multiclass}, which characterizes the sample complexity of uniform convergence via the $\ell$-graph dimension.
\begin{lemma}[Characterization of uniform convergence \cite{vapnik1974theory,daniely2011multiclass}]
\label{lemma:uniform convergence graph} 
    A $\ell$-list hypothesis class $\listC$ satisfies uniform convergence if and only if $\Gl(\listC) < \infty$. Furthermore, the sample complexity of uniform convergence of $\listC$ is $m_{\listC}^{\mathrm{UC}}(\varepsilon,\delta)
=
\Theta\!\left(
\frac{\Gl(\listC) + \log(1/\delta)}{\varepsilon^2}
\right)$.
\end{lemma}
Combining \cref{lemma:uniform convergence graph} with \cref{eq:upper bound g by d}, we obtain the desired upper bound on the sample complexity of uniform convergence:
\[m_{\listC}^{\mathrm{UC}}(\varepsilon,\delta)
=
O\!\left(
\frac{\ell d \log(\ell dk)+ \log(1/\delta)}{\varepsilon^2}
\right),\] which completes the proof.
\end{proof}

\section*{Acknowledgments}

Steve Hanneke acknowledges support by grant no.\ 2024243 from the United States - Israel Binational Science Foundation (BSF).
Shay Moran is a Robert J.\ Shillman Fellow; he acknowledges support by Israel PBC-VATAT, by the Technion Center for Machine Learning and Intelligent Systems (MLIS), and by the European Union (ERC, GENERALIZATION, 101039692). Views and opinions expressed are, however, those of the author(s) only and do not necessarily reflect those of the European Union or the European Research Council Executive Agency. Neither the European Union nor the granting authority can be held responsible for them.






\clearpage
\bibliography{References}

\clearpage
\appendix
\section{A Combinatorial Proof for the Case $\ell=1$}\label{sec:nitay}

In this appendix, we give a purely combinatorial proof of the bound
\[
|\mathcal{H}| \le 1 + n(k-1)
\]
for every class $\mathcal{H} \subseteq [k]^n$ whose $\ell$-DS dimension equals $1$ with $\ell=1$.

We are grateful to Nitay Amiel, who discovered the following purely combinatorial proof while he was an undergraduate student at the Technion.
\begin{proof}
We argue by induction on $n$.

\medskip
\noindent
\textbf{Base case: $n=1$.}
If $\mathcal{H} \subseteq [k]$ then clearly $|\mathcal{H}| \le k$.
Since $1 + 1\cdot(k-1)=k$, the bound holds.

\medskip
\noindent
\textbf{Induction step.}
Assume the statement holds for $n-1$. 
Let $\mathcal{H} \subseteq [k]^n$ have $1$-DS dimension equal to $1$.
For each vector $h \in \mathcal{H}$, write $h=(u,a)$ where
\[
u = (h(1),\dots,h(n-1)) \in [k]^{n-1}
\quad\text{and}\quad
a = h(n) \in [k].
\]
For each $u \in [k]^{n-1}$ define
\[
S(u) = \{ a \in [k] : (u,a) \in \mathcal{H} \}.
\]
Thus $S(u)$ is the set of extensions of $u$ in the $n$-th coordinate.
Now, partition according to the size of $S(u)$ and define
\[
\mathcal{H}_i = \{ u \in [k]^{n-1} : |S(u)| = i \}.
\]
Then
\[
|\mathcal{H}| = \sum_{i=1}^{k} i\,|\mathcal{H}_i|.
\]
Let
\[
\mathcal{H}_\exists = \bigcup_{i\ge 1} \mathcal{H}_i
\]
be the set of $(n-1)$-vectors that admit at least one extension.
Then
\[
|\mathcal{H}|
=
|\mathcal{H}_\exists|
+
\sum_{i=2}^{k} (i-1)|\mathcal{H}_i|.
\]
By the induction hypothesis,
\[
|\mathcal{H}_\exists|
\le
1 + (n-1)(k-1).
\]
It remains to bound the second term.
Consider the bipartite graph whose left vertices are~\(\bigcup_{i\ge 2} \mathcal{H}_i,\)
whose right vertices are $[k]$, and whose edges are
\[
\{(u,a) : (u,a) \in \mathcal{H}\}.
\]
We claim that this graph contains no cycle.
Indeed, suppose there were a cycle
\[
u_1 \to a_1 \to u_2 \to a_2 \to \cdots \to u_m \to a_m \to u_1.
\]
Since $u_1 \neq u_2$, there exists a coordinate $j<n$ such that
$u_1(j) \neq u_2(j)$. 
Restricting $\mathcal{H}$ to the coordinates $\{j,n\}$ then produces a 2-dimensional pseudo-cube, contradicting that the $1$-DS dimension equals $1$.

Therefore the graph is acyclic, and hence \(|E| \le |V|-1.\)
Now,
\[
|E| = \sum_{i=2}^{k} i\,|\mathcal{H}_i|,
\qquad
|V| = \sum_{i=2}^{k} |\mathcal{H}_i| + k.
\]
Thus
\[
\sum_{i=2}^{k} i\,|\mathcal{H}_i|
\le
\sum_{i=2}^{k} |\mathcal{H}_i| + k - 1,
\]
which implies
\[
\sum_{i=2}^{k} (i-1)|\mathcal{H}_i|
\le
k - 1.
\]
Combining the bounds,
\[
|\mathcal{H}|
=
|\mathcal{H}_\exists|
+
\sum_{i=2}^{k} (i-1)|\mathcal{H}_i|
\le
\big(1+(n-1)(k-1)\big) + (k-1)
=
1+n(k-1).
\]
This completes the proof.
\end{proof}
\end{document}